\documentclass[11pt,letterpaper]{article}
\usepackage[T1]{fontenc}
\usepackage[utf8]{inputenc}
\usepackage{lmodern}
\usepackage[margin=1in]{geometry}
\usepackage{microtype}
\usepackage[hyphens]{url}
\usepackage{graphicx}
\usepackage[round,authoryear]{natbib}
\usepackage[font=small,labelfont=bf]{caption}
\usepackage{algorithm}
\usepackage{algorithmic}
\usepackage{float}
\usepackage{placeins}
\usepackage{booktabs}
\usepackage{enumitem}
\usepackage{amsmath,amsfonts,amssymb,amsthm}
\usepackage{bm}
\usepackage{mathtools}
\usepackage[hidelinks]{hyperref}
\hypersetup{
  pdftitle={Posterior-Averaged Bellman Flow Matching for Distributional Critics},
  pdfauthor={Zeqiao Li, Yijing Wang, Haoyu Wang, Zheng Li, Zhiqiang Zuo}
}
\newtheorem{proposition}{Proposition}
\newtheorem{lemma}{Lemma}
\newtheorem{theorem}{Theorem}
\newtheorem{corollary}{Corollary}
\theoremstyle{definition}

\newtheorem{remark}{Remark}

\DeclareMathOperator{\Var}{Var}
\DeclareMathOperator{\Cov}{Cov}
\DeclareMathOperator{\ESS}{ESS}

\newcommand{\E}{\mathbb{E}}

\newcommand{\T}{\mathcal{T}}

\newcommand{\ours}{PA-BFM}

\title{Posterior-Averaged Bellman Flow Matching for Distributional Critics}
\author{
\begin{minipage}{\dimexpr\textwidth-2\tabcolsep\relax}
\centering\normalsize
Zeqiao Li\textsuperscript{\rm 1},
    Yijing Wang\textsuperscript{\rm 1},
    Haoyu Wang\textsuperscript{\rm 1,\rm 2},\\
    Zheng Li\textsuperscript{\rm 1},
    Zhiqiang Zuo\textsuperscript{\rm 1}
\\[6pt]
\small
\textsuperscript{\rm 1}Tianjin Key Laboratory of Intelligent Unmanned Swarm Technology and System,\\
    School of Electrical and Information Engineering, Tianjin University, Tianjin 300072, China\\
    \textsuperscript{\rm 2}Key Laboratory of System Control and Information Processing,\\
    Ministry of Education of China, Shanghai 200240, China\\
    \{lizeqiao, yjwang, why2014, zhengl, zqzuo\}@tju.edu.cn
\end{minipage}
}
\date{}

\begin{document}

\maketitle

\begin{abstract}
Distributional critics model the full return distribution, providing richer
signals for exploration and risk-sensitive control.
Recent work has shown that flow-based critics can model these distributions
flexibly and provide effective supervision across flow time for
temporal-difference learning.
However, unlike standard flow matching with fixed data endpoints, a flow
critic is trained on Bellman targets recursively generated by a learned target
critic under the current policy.
Existing objectives therefore either use noisy single-endpoint supervision or
rely on the learned flow itself for additional guidance. Our key observation is that an intermediate bridge state may be compatible
with several Bellman endpoints, so its velocity should average their directions
under the corresponding posterior.
We propose Posterior-Averaged Bellman Flow Matching (PA-BFM), which samples an
endpoint pool, infers the closed-form posterior over the latent anchor, and
regresses to the posterior-mean velocity.
Uniform anchor selection makes this target an exact finite-pool
Rao--Blackwellization of the single-endpoint target: for every \(K\), it
preserves the population target, cannot increase conditional variance, and
its conditional variance scales as \(O(K^{-1})\).
The endpoint pool supports conservative updates, while the twin return flows
provide critic-disagreement and unsupported-mass diagnostics. Controlled
studies show that PA-BFM reduces fixed-query target mean-squared
error by \(79.7\%\) and target-induced gradient variance by up to \(98.0\%\),
while improving return-distribution fidelity.
Across 50 OGBench tasks and four offline-to-online scenarios, PA-BFM achieves
higher aggregate performance than Value Flows.
\end{abstract}
\section{Introduction}
\label{sec:introduction}

In actor--critic reinforcement learning (RL), policy improvement depends on
the critic used to evaluate actions and construct Bellman targets. Most
methods retain only the expected return, whereas distributional RL preserves
information about stochastic outcomes, multimodality, and risk
\citep{bellemare2017distributional,dabney2018iqn,dabney2018qr,
duan2022dsac}. Recent flow-based critics represent this distribution through
a continuous transport from simple noise to return samples
\citep{dong2026valueflows,xu2026pcbf,ding2026dbc,zhong2025flowcritic,
chen2025dfc}. Flow-time supervision further provides dense signals for
changing temporal-difference targets
\citep{lipman2023flow,albergo2023stochastic,liu2023flow}.

Existing flow critics use two complementary forms of Bellman supervision.
Endpoint matching is anchored to a sampled Bellman return but uses one
continuation, whereas pathwise matching provides denser supervision through
another learned successor field.
Value Flows combines these objectives with target networks and confidence
weighting, while Path-Coupled Bellman Flows (PCBF) introduces
source-consistent geometry and a learned-field control variate
\citep{dong2026valueflows,xu2026pcbf}.
Related bridge, distillation, and architectural corrections retain the same
tension between an anchored but noisy target and guidance from a learned flow
\citep{ding2026dbc,chen2025dfc,agrawalla2025floq}.

We recast this tension as a conditional estimation problem.
At a fixed intermediate bridge state, multiple Bellman endpoints can be
compatible with the same observation.
Each compatible endpoint implies a candidate velocity, and the population
bridge velocity is the posterior average of these endpoint directions.
The uncertainty is therefore not whether a Bellman endpoint law exists, but
which compatible endpoint explains the observed intermediate state.

Our goal is to estimate this posterior velocity directly from Bellman endpoint
samples. This estimator is particularly natural in distributional RL:
conditioned on a replay transition, the target policy and return flow directly
sample the conditional Bellman endpoint law. For the one-dimensional linear
Gaussian critic bridge, candidate likelihoods are closed form, so the finite
endpoint pool induces a tractable posterior over the latent anchor without
endpoint retrieval or an auxiliary proposal distribution. In contrast,
posterior augmentation in
generative modeling often requires retrieval or approximate candidate
construction, and policy-side reverse matching may require a separate proposal
\citep{stoica2026pafm,yang2026stablevelocity,li2026rfm}.

Based on this observation, we propose
\emph{Posterior-Averaged Bellman Flow Matching} (PA-BFM).
For each transition, PA-BFM samples a finite endpoint pool, selects one
endpoint uniformly to construct the bridge state, and evaluates every
candidate using its Gaussian bridge likelihood.
The posterior-weighted direction is exactly the conditional expectation of
the single-endpoint Bellman conditional flow matching (BCFM) target.
Hence every finite pool preserves the population bridge velocity and cannot
increase conditional target variance.
We further show that its conditional variance scales as \(O(K^{-1})\) and
derive implications for gradient variability, terminal transport error, and
Bellman tracking.

The endpoint posterior supports conservative reweighting, while the twin
return flows provide disagreement and unsupported-mass diagnostics
\citep{kuznetsov2020tqc,kumar2019bear,kumar2020cql,ma2021conservative,
fujimoto2018td3}.
The main experiments use the undistorted target and retain matched flow-critic
policy extraction. Our main contributions are:
\begin{itemize}
    \item
    \textbf{Posterior-averaged Bellman supervision.}
    We formulate source-consistent endpoint matching as posterior velocity
    estimation and estimate it from \(K\) Bellman endpoints using closed-form
    bridge posteriors, without regressing to a learned successor field.

    \item
    \textbf{Exact finite-pool guarantees.}
    Uniform anchor selection makes PA-BFM an exact Rao--Blackwellization of
    BCFM. We prove finite-\(K\) unbiasedness and variance dominance,
    characterize \(K^{-1}\) conditional-variance scaling with an ESS-based bound,
    and propagate velocity error to terminal Wasserstein and Bellman tracking
    errors.

    \item
    \textbf{Extensions and empirical validation.}
    The endpoint pool supports conservative posterior tilting, while the twin
    return flows provide critic-disagreement and unsupported-mass diagnostics.
    Controlled studies, 50 OGBench tasks, and four offline-to-online scenarios
    show gains in estimation, return-law fidelity, and downstream performance.
\end{itemize}
\section{Preliminaries}
\label{sec:preliminaries}

We review distributional returns, one-dimensional flow matching, and the
source-consistent Bellman bridge used by our critic.

\subsection{MDPs and Distributional Returns}
\label{sec:prelim-mdp}

We consider a discounted Markov decision process
\(\mathcal M=(\mathcal S,\mathcal A,P,\mathcal R,\rho_0,\gamma)\), where
\(\mathcal S\) and \(\mathcal A\) are the state and action spaces,
\(P\) is the transition kernel, \(\mathcal R\) is the reward function,
\(\rho_0\) is the initial-state distribution, and \(\gamma\in(0,1)\).
In offline reinforcement learning (RL), the learner observes a fixed dataset
\(\mathcal D=\{(s,a,r,s',d)\}\), where \(d\in\{0,1\}\) indicates
termination and \(r\) is the realized reward. We use \(h\) for environment
time and \(t\in[0,1]\) for flow time.

For a policy \(\pi\), the return random variable is
\begin{equation}
    Z^\pi(s,a)
    =
    \sum_{h=0}^{\infty}\gamma^h \mathcal R(S_h,A_h),
    \qquad
    S_0=s,\quad A_0=a,
    \label{eq:return-rv}
\end{equation}
where \(S_{h+1}\sim P(\cdot\mid S_h,A_h)\) and
\(A_h\sim\pi(\cdot\mid S_h)\) for \(h\ge1\).
Its expectation is the action value
\(Q^\pi(s,a)=\E[Z^\pi(s,a)]\), whereas distributional RL models the law of
\(Z^\pi(s,a)\) itself
\citep{bellemare2017distributional,dabney2018iqn,dabney2018qr}.

For a state--action conditional return variable \(Z\), the distributional
Bellman operator is
\begin{equation}
\begin{aligned}
    (\T^\pi Z)(s,a)
    &\stackrel{d}{=}
    r+\gamma(1-d)Z(S',A'),\\
    S'&\sim P(\cdot\mid s,a),
    \qquad
    A'\sim\pi(\cdot\mid S').
\end{aligned}
    \label{eq:dist-bellman-rv}
\end{equation}
\par\indent For a nonterminal transition with reward \(r\), its density satisfies
\begin{equation}
    p_{\T^\pi Z}(z\mid s,a)
    =
    \frac{1}{\gamma}
    \E_{S',A'}
    \left[
        p_Z\!\left(
            \frac{z-r}{\gamma}
            \middle|S',A'
        \right)
    \right].
    \label{eq:pdf-bellman}
\end{equation}
\par\indent Terminal transitions place all mass at \(z=r\).
The operator is a \(\gamma\)-contraction under Wasserstein metrics and has
\(Z^\pi\) as its unique fixed point
\citep{bellemare2017distributional}.

\subsection{Flow Matching}
\label{sec:prelim-fm}

Flow matching (FM) learns a time-dependent vector field that transports a
standard Gaussian source to a target distribution \(p_1(\cdot\mid c)\).
Let \(\epsilon\sim\mathcal N(0,1)\), and let \(p_0\) denote its density. The
flow solves
\begin{equation}
    \frac{d}{dt}\phi_\theta(\epsilon\mid t,c)
    =
    v_\theta\!\left(
        \phi_\theta(\epsilon\mid t,c),t;c
    \right),
    \qquad
    \phi_\theta(\epsilon\mid0,c)=\epsilon.
    \label{eq:return-flow-ode}
\end{equation}
\par\indent Conditional flow matching (CFM) trains this field using tractable paths
\citep{lipman2023flow,albergo2023stochastic,liu2023flow}.
For independent \(\epsilon\sim\mathcal N(0,1)\),
\(y\sim p_1(\cdot\mid c)\), and
\(t\sim\operatorname{Unif}([0,1))\), define
\begin{equation}
    X_t=(1-t)\epsilon+ty,
    \qquad
    u_t(X_t,y)=y-\epsilon=\frac{y-X_t}{1-t}.
    \label{eq:linear-cfm-path}
\end{equation}
\par\indent The CFM objective is
\begin{equation}
    \mathcal L_{\mathrm{CFM}}(\theta)
    =
    \E\!\left[
        \left\|
        v_\theta(X_t,t;c)-u_t(X_t,y)
        \right\|_2^2
    \right].
    \label{eq:cfm}
\end{equation}
\par\indent Its population minimizer is
\(v^\star(x,t;c)=\E[u_t(x,Y)\mid X_t=x,c]\): the conditional average of
all path directions compatible with the observed intermediate point.
For our critic, \(c=(s,a)\), and
\(Z_\theta(s,a)\stackrel{d}{=}\phi_\theta(\epsilon\mid1,s,a)\).

\subsection{Bellman Endpoint Laws and Source-Consistent Bridges}
\label{sec:prelim-bellman-flow}

Fix \((s,a)\) and let \(\tau=(r,s',d)\) denote one transition outcome.
Let \(Y\sim\mu_\tau\) denote the abstract conditional Bellman endpoint law.
A sample from this law is generated by a target-flow return
\(Z_{\bar\theta}(s',a';\epsilon')\):
\begin{equation}
\begin{aligned}
    Y
    &=r+\gamma(1-d)
      Z_{\bar\theta}(s',A';\epsilon'),\\
    \mu_\tau&=\operatorname{Law}(Y\mid\tau),
    \qquad
    A'\sim\pi(\cdot\mid s'),\quad
    \epsilon'\sim\mathcal N(0,1).
\end{aligned}
    \label{eq:conditional-bellman-endpoint-law}
\end{equation}
\par\indent In our implementation, \(\mu_\tau\) is the law of the shared-noise mean of
the two target critics; the explicit construction is given in
Appendix~\ref{app:critic-update}. An endpoint pool averages the
successor-action and target-flow randomness
within one transition; replay sampling performs the outer average over
\(\tau\mid s,a\).

The Bellman endpoint law does not uniquely determine an intermediate path.
For a nonterminal transition (\(d=0\)), applying the Bellman affine map to
each slice of a successor flow gives
\(X_t^{\mathrm{sl}}=r+\gamma X_t'\). If \(X_t'\) has density
\(p_t'(\cdot\mid s',A')\) and velocity \(v_t'\), then, with
\(\xi=(x-r)/\gamma\),
\begin{align}
    q_t^{\mathrm{sl}}(x)
    &=\frac1\gamma
      \E_{A'}[p_t'(\xi\mid s',A')],
    \label{eq:bellman-density-path}\\
    v_t^{\mathrm{sl}}(x)
    &=\gamma
      \frac{\E_{A'}[p_t'(\xi\mid s',A')v_t'(\xi\mid s',A')]}
           {\E_{A'}[p_t'(\xi\mid s',A')]}.
    \label{eq:vf-bellman-velocity}
\end{align}
\par\indent The factor \(\gamma\) follows from
\(\dot X_t^{\mathrm{sl}}=\gamma\dot X_t'\).
Although this path reaches the correct Bellman endpoint, its source is
\(\mathcal N(r,\gamma^2)\), rather than the common critic source
\(\mathcal N(0,1)\).
For \(d=1\), the endpoint law is the point mass at \(Y=r\).

We therefore use the source-consistent bridge
\begin{equation}
\begin{aligned}
    X_t&=(1-t)\epsilon+tY,
    \qquad \epsilon\perp Y,\\
    \epsilon&\sim\mathcal N(0,1),
    \qquad Y\sim\mu_\tau.
\end{aligned}
    \label{eq:source-consistent-bellman-bridge}
\end{equation}
\par\indent If the successor uses the linear path
\(X_t'=(1-t)\epsilon+tX_1'\) with the same source noise, then
\begin{equation}
    X_t
    =
    X_t^{\mathrm{sl}}
    +(1-t)\big[(1-\gamma)\epsilon-r\big].
    \label{eq:path-residual-correction}
\end{equation}
\par\indent Thus the two paths share the endpoint at \(t=1\), but not their intermediate
marginals.

Using the standard Gaussian density \(p_0\), define
\begin{equation}
    \ell_t(x\mid y)
    =
    \frac{1}{1-t}
    p_0\!\left(\frac{x-ty}{1-t}\right),
    \qquad
    u_t(x,y)=\frac{y-x}{1-t}.
    \label{eq:bridge-likelihood-velocity}
\end{equation}
\par\indent The transition-conditional bridge velocity is
\begin{equation}
    v_\tau^\star(x,t)
    =
    \E\!\left[
        u_t(x,Y)
        \middle|
        X_t=x,\tau
    \right].
    \label{eq:transition-bridge-velocity}
\end{equation}
\par\indent The full state--action velocity is recovered by the outer transition mixture,
\begin{equation}
    v^\star(x,t\mid s,a)
    =
    \E\!\left[
        v_\tau^\star(x,t)
        \middle|
        X_t=x,s,a
    \right].
    \label{eq:transition-posterior-decomposition}
\end{equation}
\par\indent Hence the endpoint pool estimates the velocity conditional on one transition,
while replay sampling accounts for environment-transition randomness.
\section{Method}
\label{sec:method}

We use PA-BFM for the complete algorithm and PA-BCFM for its
posterior-averaged critic-training objective.

\subsection{Bellman Objectives and Posterior Velocity}
\label{sec:from-bellman-flow-to-posterior}

Bellman target noise is recursively reused in later temporal-difference
updates. Unlike standard FM with fixed data endpoints, a flow critic draws
from a learned Bellman law under the current policy, so error can re-enter
later targets. We seek an endpoint-anchored estimator that averages
continuation uncertainty already present in this law.
For a sampled transition \(\tau=(r,s',d)\), draw a Bellman endpoint
\begin{equation}
\begin{aligned}
    Y
    &=r+\gamma(1-d)Z_{\bar\theta}(s',A';\epsilon'),\\
    A'\sim\pi(\cdot\mid s'),
    &\qquad \epsilon'\sim\mathcal N(0,1),
\end{aligned}
    \label{eq:bellman-endpoint}
\end{equation}
and construct the source-consistent bridge
\(X_t=(1-t)\epsilon+tY\), whose supervision direction is
\(u_t(X_t,Y)=(Y-X_t)/(1-t)=Y-\epsilon\).
Bellman conditional flow matching (BCFM) minimizes
\begin{equation}
    \mathcal L_{\mathrm{BCFM}}(\theta)
    =
    \E\!\left[
        \left\|
        v_\theta(X_t,t;s,a)-u_t(X_t,Y)
        \right\|_2^2
    \right].
    \label{eq:bcfm}
\end{equation}
\par\indent Conditioned on \(X_t=x\), its population target is
\begin{equation}
    v_\tau^\star(x,t)
    =
    \E\!\left[
        u_t(x,Y)
        \middle|
        X_t=x,\tau
    \right].
    \label{eq:posterior-bellman-velocity}
\end{equation}
\par\indent BCFM is conditionally unbiased, but \(X_t=x\) does not identify which
compatible endpoint was sampled. Using one continuation therefore introduces
avoidable label variance.
This is a latent-identity problem: after observing the bridge point, the
missing quantity is which compatible continuation produced it.

Value Flows additionally uses distributional conditional flow matching
(DCFM) to impose pathwise consistency
\citep{dong2026valueflows}. With
\(X_t'=\phi_{\bar\theta}(\epsilon'\mid t,s',A')\), its practical objective is
\begin{equation}
\begin{aligned}
    \mathcal L_{\mathrm{DCFM}}^{\mathrm{VF}}(\theta)
    &=\E\!\left[\|\Delta_t^{\mathrm{VF}}\|_2^2\right],\\
    \Delta_t^{\mathrm{VF}}
    &=v_\theta(r+\gamma X_t',t;s,a)
      -v_{\bar\theta}(X_t',t;s',A').
\end{aligned}
    \label{eq:vf-dcfm}
\end{equation}
\par\indent DCFM provides dense learned-field supervision, but it is not an estimator of
the source-consistent posterior target in
Eq.~\eqref{eq:posterior-bellman-velocity}.
The corresponding slice-wise affine geometry would instead require
\begin{equation}
    v_\theta(r+\gamma x',t;s,a)
    \approx
    \gamma v_{\bar\theta}(x',t;s',A'),
    \label{eq:scaled-dcfm}
\end{equation}
and follows the source-mismatched path in
Eq.~\eqref{eq:vf-bellman-velocity}.
Because both sides of Eq.~\eqref{eq:vf-dcfm} are learned, degenerate shared
fields may also satisfy its residual.
DCFM is therefore dense but self-referential, whereas BCFM is anchored but
single-sample. Bridge corrections and critic-specific architectures can
strengthen this supervision, but preserve the distinction between sampled
endpoint anchors and guidance from a learned field
\citep{ding2026dbc,chen2025dfc,agrawalla2025floq}.
PCBF instead corrects the single-endpoint target with a learned-field control
variate and a tunable coefficient \citep{xu2026pcbf}.
Thus BCFM samples one continuation, PCBF corrects it with a learned field, and
PA-BFM conditions the same target on a finite endpoint pool.
This comparison is estimator-level: PA-BFM changes how the Bellman direction
is formed while leaving the target policy, return-flow parameterization, and
policy-extraction protocol unchanged.
\subsection{Posterior-Averaged Bellman Flow Matching}
\label{sec:pa-bcfm}

PA-BCFM implements the conditional averaging suggested above in four steps:
it samples a Bellman endpoint pool, selects one latent anchor to form the
bridge point, infers the anchor posterior, and averages the corresponding BCFM
directions.
Unlike generic posterior augmentation, no retrieval model is needed: the
target return flow already supplies candidates from the relevant
transition-conditional law, and the linear Gaussian bridge supplies their
likelihoods analytically. The pool exposes uncertainty about compatible
continuations without adding another learned correction field.

\par\indent For one sampled transition, PA-BCFM draws
\begin{equation}
    \mathbf Y=(Y_1,\ldots,Y_K),
    \qquad
    Y_j\overset{\mathrm{i.i.d.}}{\sim}\mu_\tau,
    \quad j=1,\ldots,K.
    \label{eq:pa-endpoint-pool}
\end{equation}
\par\indent Each endpoint is generated by the target policy and target return flow as in
Eq.~\eqref{eq:bellman-endpoint}. In practice, \(\mu_\tau\) is instantiated by
the shared-noise mean of the twin target critics; this does not change the
estimator below. All candidates share the sampled transition but retain
independent successor-action and return-noise draws, so increasing \(K\)
targets continuation uncertainty rather than replay-transition uncertainty.

\par\indent It then chooses one candidate uniformly to construct the observed bridge
point,
\begin{equation}
\begin{aligned}
    I&\sim\operatorname{Unif}\{1,\ldots,K\},
       \qquad \epsilon\sim\mathcal N(0,1),\\
    x_t&=(1-t)\epsilon+tY_I.
\end{aligned}
    \label{eq:pa-anchor-bridge}
\end{equation}
\par\indent The index \(I\) is the unobserved identity of the endpoint that generated
\(x_t\). Sampling the anchor from the same pool gives a finite candidate model
in which the posterior responsibility of every endpoint can be evaluated
exactly.

\par\indent Given the bridge point, a candidate \(Y_j\) explains \(x_t\) through the source value
\((x_t-tY_j)/(1-t)\). Bayes' rule therefore gives
\begin{equation}
\begin{aligned}
    w_j(x_t,t)
    &=
    \frac{\ell_t(x_t\mid Y_j)}
         {\sum_{\ell=1}^{K}\ell_t(x_t\mid Y_\ell)}\\
    &=
    \Pr\!\left(
        I=j
        \middle|
        X_t=x_t,\mathbf Y,\tau,t
    \right).
\end{aligned}
    \label{eq:pa-posterior-weight}
\end{equation}
\par\indent Thus \(w_j\) is the exact posterior responsibility of candidate \(Y_j\), not
a heuristic similarity score. Uniform averaging ignores this compatibility by
assigning every candidate the same responsibility. Early in the bridge, many
endpoints can explain \(x_t\) and the weights are broad; near the endpoint,
compatibility becomes sharper. The endpoint-pool effective sample size below
measures this time-dependent concentration.

\par\indent Finally, define the candidate directions and posterior target by
\begin{equation}
\begin{aligned}
    u_j&=u_t(x_t,Y_j)=\frac{Y_j-x_t}{1-t},\\
    \widehat u_K(x_t,t)&=\sum_{j=1}^{K}w_j(x_t,t)u_j.
\end{aligned}
    \label{eq:pa-posterior-velocity}
\end{equation}
\par\indent PA-BCFM trains the twin online critics toward this shared target:
\begin{equation}
\begin{aligned}
    \mathcal L_{\mathrm{PA\text{-}BCFM}}
    &=\frac12\sum_{c=1}^{2}\E\!\Bigl[\\[-2pt]
    &\qquad
        \Bigl\|
        v_{\theta_c}(x_t,t;s,a)
        -\operatorname{sg}[\widehat u_K(x_t,t)]
        \Bigr\|_2^2
    \Bigr],
\end{aligned}
    \label{eq:pa-bcfm-loss}
\end{equation}
where \(\operatorname{sg}\) denotes stop-gradient.
Replay sampling supplies the outer transition mixture in
Eq.~\eqref{eq:transition-posterior-decomposition}.
Thus PA-BCFM modifies only critic supervision. The policy proposal, action
evaluation, and target-network update remain those of the matched flow-critic
implementation.

The key point is that posterior averaging changes only the target noise, not
the target mean. The following result holds for the practical finite pool.

\begin{theorem}[Finite-pool Rao--Blackwellization]
\label{thm:pa-bcfm-estimator}
Assume that the bridge normalizer is positive and the endpoint direction has
a finite conditional second moment.
When \(K=1\), \(\widehat u_1=u_I\), so PA-BCFM exactly recovers BCFM.
For every \(K\ge1\),
\begin{equation}
\begin{aligned}
    \widehat u_K
    &=
    \E\!\left[
        u_I
        \middle|
        X_t=x_t,\mathbf Y,\tau,t
    \right],\\
    \E[\widehat u_K\mid X_t=x_t,\tau,t]
    &=v_\tau^\star(x_t,t).
\end{aligned}
    \label{eq:pa-rb-unbiasedness}
\end{equation}
\par\indent Moreover,
\begin{equation}
    \Var(\widehat u_K\mid X_t,\tau,t)
    \le
    \Var(u_I\mid X_t,\tau,t).
    \label{eq:pa-rb-variance-order}
\end{equation}
\par\indent Thus every finite pool preserves the population target and cannot increase
conditional target variance.
\end{theorem}

\par\indent Appendix~\ref{app:anchor-reuse} gives the corresponding weighted
total-variance equality. PA-BCFM removes posterior dispersion among compatible
endpoint directions while leaving the conditional mean unchanged.
The pool supplies additional information by conditioning away the sampled
continuation identity, not by changing the Bellman endpoint law.
Equivalently, for any prediction \(v\),
\begin{equation}
    \E\!\left[(v-\widehat u_K)^2\mid X_t,\tau,t\right]
    =
    (v-v_\tau^\star)^2
    +\Var(\widehat u_K\mid X_t,\tau,t),
    \label{eq:pa-population-field-risk}
\end{equation}
so every finite pool has the same population regression minimizer.

Under the moment conditions in Appendix~\ref{app:asymptotic-ess}, the
remaining large-pool variance obeys
\begin{equation}
    \Var(\widehat u_K\mid x,t,\tau)
    =
    \frac{\sigma_\tau^2(x,t)}{K}+o(K^{-1}),
    \label{eq:pa-snis-covariance}
\end{equation}
where
\begin{equation}
    \sigma_\tau^2(x,t)
    =
    \frac{
        \E_{Y\sim\mu_\tau}\!\left[
            \ell_t(x\mid Y)^2
            \bigl(u_t(x,Y)-v_\tau^\star(x,t)\bigr)^2
        \right]
    }{
        \left(
            \E_{Y\sim\mu_\tau}[\ell_t(x\mid Y)]
        \right)^2
    }.
    \label{eq:pa-rb-asymptotic-covariance}
\end{equation}
\par\indent The effective sample size
\begin{equation}
    \ESS(\mathbf w)
    =
    \frac{1}{\sum_{j=1}^{K}w_j^2}
    \label{eq:pa-ess-main}
\end{equation}
measures posterior concentration. Appendix~\ref{app:asymptotic-ess} shows that
the normalized ESS converges to a transition-dependent concentration factor
and controls the constant in the asymptotic variance bound. Larger pools help
most when several candidates retain nonnegligible posterior mass; a larger
nominal \(K\) cannot compensate for a fully concentrated posterior.

For the squared loss, if \(g_1\) is the BCFM gradient and \(g_K\) the
PA-BCFM gradient at the same bridge point, then
\begin{equation}
\begin{aligned}
    g_K
    &=\E[g_1\mid X_t=x_t,\mathbf Y,\tau,t],\\
    \Cov(g_K\mid X_t,\tau,t)
    &\preceq\Cov(g_1\mid X_t,\tau,t).
\end{aligned}
    \label{eq:pa-gradient-rb}
\end{equation}
\par\indent Thus the same Rao--Blackwell ordering applies directly to target-induced
gradient variability.

\begin{corollary}[Training and transport error]
\label{cor:pa-training-transport}
Under the smoothness, Polyak--\L{}ojasiewicz, and flow-Lipschitz assumptions
in Appendix~\ref{app:flow-stability}, the pool-induced contribution to the SGD
error floor is \(O((B_{\mathrm{mb}}K)^{-1})\), and its contribution to terminal
\(W_1\) error is \(O((B_{\mathrm{mb}}K)^{-1/2})\). For a fixed policy,
Bellman contraction amplifies the one-step error by at most
\(1/(1-\gamma)\).
\end{corollary}

The complete constants and the changing-policy tracking recursion are given
in Appendix~\ref{app:flow-stability}.
\newcommand{\offlineaggtable}{%
\begin{table}[t]
    \centering
    \small
    \setlength{\tabcolsep}{2.25pt}
    \resizebox{\linewidth}{!}{%
\begin{tabular}{lcccccccc}
        \toprule
        & \multicolumn{2}{c}{Gaussian Policies}
        & \multicolumn{2}{c}{Flow Policies}
        & \multicolumn{4}{c}{Distributional Critics} \\
        \cmidrule(lr){2-3}
        \cmidrule(lr){4-5}
        \cmidrule(lr){6-9}
        Domain
        & IQL & ReBRAC
        & IFQL & FQL & IQN
        & Value Flows & PCBF & \ours{} \\
        \midrule
        \texttt{cube-double-play (5 tasks)}
        & $6 \pm 2$
        & $12 \pm 3$
        & $14 \pm 5$
        & $29 \pm 6$
        & $42 \pm 8$
        & $69 \pm 4$
        & $71 \pm 5$
        & $\mathbf{79 \pm 4}$ \\

        \texttt{cube-triple-play (5 tasks)}
        & $1 \pm 1$
        & $0 \pm 0$
        & $0 \pm 0$
        & $4 \pm 2$
        & $6 \pm 0$
        & $\mathbf{14 \pm 3}$
        & $4 \pm 1$
        & $\mathbf{14 \pm 2}$ \\

        \texttt{puzzle-3x3-play (5 tasks)}
        & $9 \pm 3$
        & $22 \pm 2$
        & $19 \pm 1$
        & $30 \pm 4$
        & $15 \pm 1$
        & $87 \pm 13$
        & $53 \pm 3$
        & $\mathbf{93 \pm 3}$ \\

        \texttt{puzzle-4x4-play (5 tasks)}
        & $7 \pm 2$
        & $14 \pm 3$
        & $25 \pm 8$
        & $17 \pm 5$
        & $27 \pm 4$
        & $27 \pm 4$
        & $30 \pm 4$
        & $\mathbf{52 \pm 6}$ \\

        \texttt{scene-play (5 tasks)}
        & $28 \pm 3$
        & $41 \pm 7$
        & $30 \pm 4$
        & $56 \pm 2$
        & $40 \pm 1$
        & $59 \pm 4$
        & $54 \pm 4$
        & $\mathbf{78 \pm 2}$ \\

        \midrule
        \texttt{visual-antmaze-medium-navigate (5 tasks)}
        & $84 \pm 5$
        & $87 \pm 4$
        & $87 \pm 2$
        & $38 \pm 4$
        & $74 \pm 4$
        & $75 \pm 10$
        & $75 \pm 2$
        & $\mathbf{91 \pm 3}$ \\

        \texttt{visual-antmaze-teleport-navigate (5 tasks)}
        & $6 \pm 4$
        & $4 \pm 1$
        & $10 \pm 4$
        & $5 \pm 2$
        & $4 \pm 2$
        & $13 \pm 4$
        & $14 \pm 4$
        & $\mathbf{22 \pm 5}$ \\

        \texttt{visual-cube-double-play (5 tasks)}
        & $11 \pm 6$
        & $1 \pm 1$
        & $2 \pm 2$
        & $6 \pm 1$
        & $1 \pm 0$
        & $13 \pm 2$
        & $3 \pm 1$
        & $\mathbf{17 \pm 5}$ \\

        \texttt{visual-puzzle-3x3-play (5 tasks)}
        & $2 \pm 3$
        & $20 \pm 1$
        & $21 \pm 0$
        & $20 \pm 1$
        & $19 \pm 1$
        & $\mathbf{23 \pm 2}$
        & $18 \pm 1$
        & $20 \pm 1$ \\

        \texttt{visual-scene-play (5 tasks)}
        & $26 \pm 5$
        & $28 \pm 5$
        & $21 \pm 2$
        & $41 \pm 4$
        & $41 \pm 6$
        & $43 \pm 7$
        & $30 \pm 3$
        & $\mathbf{50 \pm 8}$ \\

        \bottomrule
    \end{tabular}}
    \caption{Offline evaluation on OGBench.
    Each row averages five tasks; state-based and image-based results use eight
    and four seeds. Uncertainty is the arithmetic mean of the five task-level
    standard deviations. Matched flow critics share policy-side components
    while retaining their original critic supervision. Bold denotes the
    highest displayed mean in each domain.}
    \label{tab:offline-eval-agg}
\end{table}
}
\subsection{Conservative Learning and Critic Diagnostics}
\label{sec:endpoint-pool-interface}

Before each critic update, PA-BCFM has formed the weighted endpoint pool
\(\{(Y_j,w_j)\}_{j=1}^K\). This pool supports conservative reweighting and a
training-time disagreement signal. Separate state--action diagnostics are
computed directly from the twin return flows.

\par\indent Let \(g:\mathbb R\to[0,\infty)\) favor the returns retained by a desired
distortion. The reweighted posterior target is
\begin{align}
    w_j^g
    &=
    \frac{w_jg(Y_j)}{\sum_{\ell=1}^{K}w_\ell g(Y_\ell)},
    \label{eq:distorted-weight}\\
    \widehat u_K^g
    &=
    \sum_{j=1}^{K}w_j^g u_j.
    \label{eq:distorted-pa-target}
\end{align}
\par\indent The exact finite-pool Rao--Blackwell identity applies to the undistorted
target \(g\equiv1\), but generally not to a distorted target \(g\ne1\).
    A nonincreasing \(g\) downweights high returns and yields a pessimistically
    distorted target; hard upper-tail truncation and the soft tilt
    \(g_\eta(y)=\exp(-\eta y)\) are two examples
    \citep{kuznetsov2020tqc,kumar2019bear,kumar2020cql,
    ma2021conservative}. If the population distortion is
\(L_g\)-Lipschitz in \(W_p\), its Bellman operator is
\(\gamma L_g\)-Lipschitz and is contractive when \(\gamma L_g<1\).
Appendix~\ref{app:distorted-laws} gives the population definitions and proofs.

\par\indent For \(c\in\{1,2\}\), the target critics share successor actions and return
noises:
\begin{align}
    Y_j^{(c)}
    &=r+\gamma(1-d)
      \phi_{\bar\theta_c}(\epsilon'_j\mid1,s',A'_j),
    \label{eq:twin-flow-endpoints}\\
    Y_j&=\tfrac12\bigl(Y_j^{(1)}+Y_j^{(2)}\bigr),\nonumber\\
    D_{\mathrm{pool}}(\tau,x_t,t,\mathbf Y)
    &=\sum_{j=1}^{K}w_j^0\nonumber\\[-2pt]
    &\quad{}
      \bigl(Y_j^{(1)}-Y_j^{(2)}\bigr)^2.
    \label{eq:twin-flow-disagreement}
\end{align}
\par\indent Here \(A'_j\sim\pi(\cdot\mid s')\) and
\(\epsilon'_j\sim\mathcal N(0,1)\), and \(w_j^0\) denotes the undistorted
posterior weight obtained with \(g\equiv1\). Thus the statistic is defined
before any adaptive reweighting. For task \(e\), let
\(\Delta_e=q_{.95}^{\mathcal D,e}-q_{.05}^{\mathcal D,e}\) be the fixed
offline return-to-go range. The stop-gradient adaptive heuristic is
\begin{equation}
\begin{aligned}
    \eta(\tau,x_t,t)
    &=\eta_0\operatorname{sg}\!\left[
        \operatorname{clip}\!\left(
            \frac{\sqrt{D_{\mathrm{pool}}}}{\Delta_e},0,1
        \right)
    \right],\\
    g_{\tau,x_t,t}(y)
    &=\exp\!\left[-\eta(\tau,x_t,t)\frac{y}{\Delta_e}\right].
\end{aligned}
    \label{eq:adaptive-distortion}
\end{equation}
\par\indent For a fixed distortion \(g\), Appendix~\ref{app:distorted-laws}
establishes consistency and operator-level contraction conditions. The
disagreement-adaptive schedule in Eq.~\eqref{eq:adaptive-distortion} is a
stop-gradient practical extension and is not covered by the fixed-distortion
contraction result; its scope is clarified in
Appendix~\ref{app:adaptive-distortion-contraction}.

\par\indent For a queried action, shared return noises isolate disagreement between the
twin target return flows:
\begin{equation}
\begin{aligned}
    Z_m^{(c)}(s,a)
    &=\phi_{\bar\theta_c}(\xi_m\mid1,s,a),
      \qquad c\in\{1,2\},\\
    \xi_m&\sim\mathcal N(0,1),
      \qquad m=1,\ldots,M_{\mathrm{diag}},\\
    D_{\mathrm{twin}}(s,a)
    &=\frac1{M_{\mathrm{diag}}}
      \sum_{m=1}^{M_{\mathrm{diag}}}
      \bigl(Z_m^{(1)}(s,a)-Z_m^{(2)}(s,a)\bigr)^2.
\end{aligned}
    \label{eq:state-action-twin-disagreement}
\end{equation}
\par\indent For each critic, flow inversion gives
\begin{align}
    \widehat F_c(z\mid s,a)
    &=\Phi\!\left(
        \phi_{\bar\theta_c}^{-1}(z\mid1,s,a)
    \right),
    \label{eq:flow-cdf}\\
    \widehat F_{\mathrm{ens}}(z\mid s,a)
    &=\frac12\sum_{c=1}^{2}\widehat F_c(z\mid s,a),\nonumber\\
    P_{\mathrm{out}}(s,a)
    &=1-\widehat F_{\mathrm{ens}}
      (q_{\mathrm{hi}}^{\mathcal D}\mid s,a).
    \label{eq:support-violation}
\end{align}
\par\indent \(P_{\mathrm{out}}\) is the twin-ensemble mass above an empirical upper-tail
threshold fixed before diagnostic evaluation. Unlike these state--action
diagnostics, \(D_{\mathrm{pool}}\) is conditioned on a training transition
and bridge point.

\subsection{Distributional Policy Improvement}
\label{sec:policy-improvement}

PA-BFM changes only the critic objective: a behavior flow proposes \(N\)
actions, each evaluated with \(M_{\mathrm{act}}\) shared return noises:
\begin{equation}
\begin{aligned}
    a_i&\sim\pi_\beta(\cdot\mid s),
      \qquad i=1,\ldots,N,\\
    \epsilon_{im}&\sim\mathcal N(0,1),
      \qquad m=1,\ldots,M_{\mathrm{act}},\\
    z_{im}^{(c)}
    &=\phi_{\theta_c}(\epsilon_{im}\mid1,s,a_i),
      \qquad c\in\{1,2\}.
\end{aligned}
    \label{eq:policy-return-samples}
\end{equation}
\par\indent The main experiments use a Monte Carlo estimate of mean return,
\begin{equation}
\begin{aligned}
    \widehat\mu_i
    &=\frac{1}{2M_{\mathrm{act}}}
      \sum_{c=1}^{2}\sum_{m=1}^{M_{\mathrm{act}}}z_{im}^{(c)},\\
    a^\star
    &=\operatorname*{arg\,max}_{1\le i\le N}\widehat\mu_i.
\end{aligned}
    \label{eq:distributional-action-selection}
\end{equation}
\par\indent Appendix~\ref{app:policy-extraction} gives the shared
candidate-selection and policy-training procedure. PA-BFM introduces no
method-specific actor-side objective; matched flow-critic methods share the
same policy-side procedure.

\offlineaggtable
\begin{figure}[t]
    \centering
    \includegraphics[width=\textwidth]
    {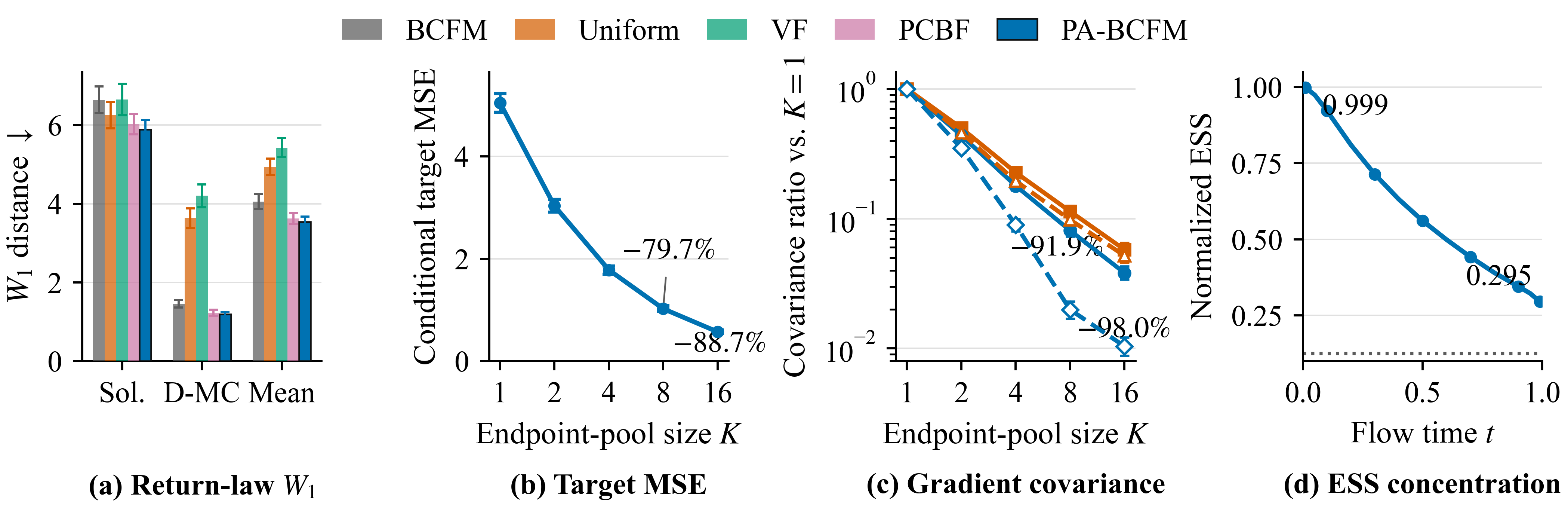}
    \caption{Controlled posterior-target evaluation.
    \emph{(a)} Return-law $W_1$; \emph{(b)} conditional target MSE;
    \emph{(c)} target-induced gradient trace-covariance ratio; and \emph{(d)}
    normalized $\ESS(\mathbf w)/K$ at $K=8$. Sol./D-MC and VF denote
    Solitaire Dice/Discrete-MC and Value Flows; solid/dashed in (c) follow the
    same environment order, and dotted in (d) is $1/K$. Error bars are 95\%
    confidence intervals; (b)--(c) use exact bridge-posterior anchors at fixed
    $(\tau,t,x_t)$.}
    \label{fig:controlled-mrp}
\end{figure}

\section{Experiments}
\label{sec:experiments}

We separate estimator-level evaluation from downstream control.
\textbf{(Q1)} Controlled processes test posterior-target error, gradient
variability, and return-law fidelity. \textbf{(Q2)} Offline and
offline-to-online benchmarks test downstream policy performance.
\textbf{(Q3)} Ablations measure posterior weighting, pool size, and cost.
\textbf{(Q4)} Diagnostics test reliability and pessimistic trade-offs.

The controlled suite comprises a Bernoulli-return process, Solitaire Dice,
and a discrete Monte Carlo chain. Downstream evaluation covers 50 OGBench
tasks and four offline-to-online scenarios
\citep{park2025ogbenchbenchmark}. Baselines span scalar, flow-policy, and
distributional critic families; RLPD appears only offline-to-online
\citep{kostrikov2021,tarasov2023revisitingminimalist,park2025flow,
bellemare2017distributional,dabney2018iqn,ma2021conservative,
dong2026valueflows,xu2026pcbf,ball2023rlpd}.

\paragraph{Methods and evaluation.}
We use a common JAX evaluation pipeline. BCFM, Value Flows, PCBF, and
PA-BFM follow a matched code path and share policy-side and
optimization components while retaining their critic objectives and
flow-time supervision, isolating posterior averaging within this group.
Hyperparameters are tuned once per domain under a common budget; state and
image tasks use eight and four seeds. \(K=8\) except in controlled and
ablation studies. Full implementation, tuning, and evaluation details appear
in
Appendices~\ref{app:additional-experiments}--\ref{app:experiment-details}.

\subsection{Controlled Posterior-Target Estimation}
\label{sec:controlled-mrp}

The controlled suite separates conditional target estimation from
trained-critic fidelity. Bernoulli fixes the configuration, while Solitaire
Dice and Discrete-MC are used for evaluation. For each fixed bridge query
$(\tau,t,x_t)$, every replicate independently resamples the anchor from the
exact bridge posterior and draws the remaining $K-1$ endpoints from
$\mu_\tau$. We compare
BCFM, uniform aggregation, and PA-BCFM for $K\in\{1,2,4,8,16\}$; the
learned-flow study includes Value Flows and PCBF with the fixed coefficient
$\lambda_{\mathrm{cv}}=0.3$~\citep{dong2026valueflows,xu2026pcbf}.

Figure~\ref{fig:controlled-mrp}(b--d) reports fixed-query estimates.
Signed-bias 95\% CIs contain zero for all $K$. Target MSE falls to
$1.022\,[0.97,1.08]$ at $K=8$ and $0.570\,[0.54,0.61]$ at $K=16$,
reductions of $79.7\%$ and $88.7\%$ relative to $K=1$. At $K=8$, the
target-induced gradient trace-covariance ratios are
$0.081\,[0.073,0.090]$ on Solitaire Dice and $0.020\,[0.017,0.023]$ on
Discrete-MC. Normalized $\ESS(\mathbf w)/K$ falls from $0.999$ at $t=0.01$ to $0.295$ at
$t=0.99$, indicating broad early averaging and sharper late selection.
Figure~\ref{fig:controlled-mrp}(a) reports seed-level return-law fidelity:
PA-BCFM has the lowest mean $W_1$ on both processes and overall,
$3.538\,[3.41,3.67]$; its paired difference from PCBF is
$-0.085\,[-0.156,-0.014]$.

\setcounter{topnumber}{3}
\begin{figure}[ht]
    \centering
    \includegraphics[width=0.68\linewidth]
    {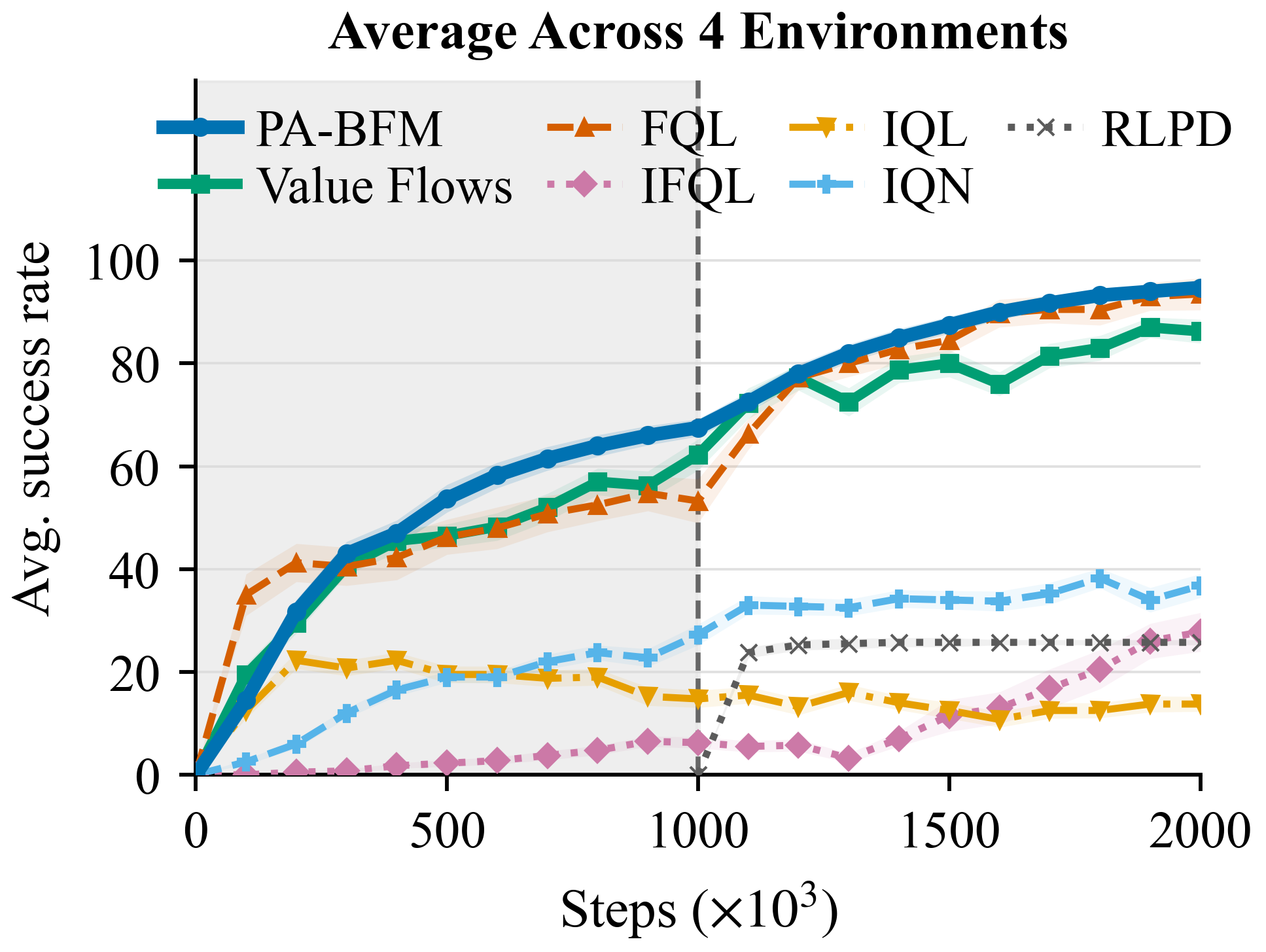}
    \caption{Offline-to-online evaluation.
    Mean success and one-standard-deviation bands over eight seeds on four
    OGBench scenarios. Gray marks one million offline optimizer updates,
    followed by one million online environment steps.}
    \label{fig:o2o-aggregate}
\end{figure}

\begin{figure}[t]
    \centering
    \includegraphics[width=0.68\linewidth]
    {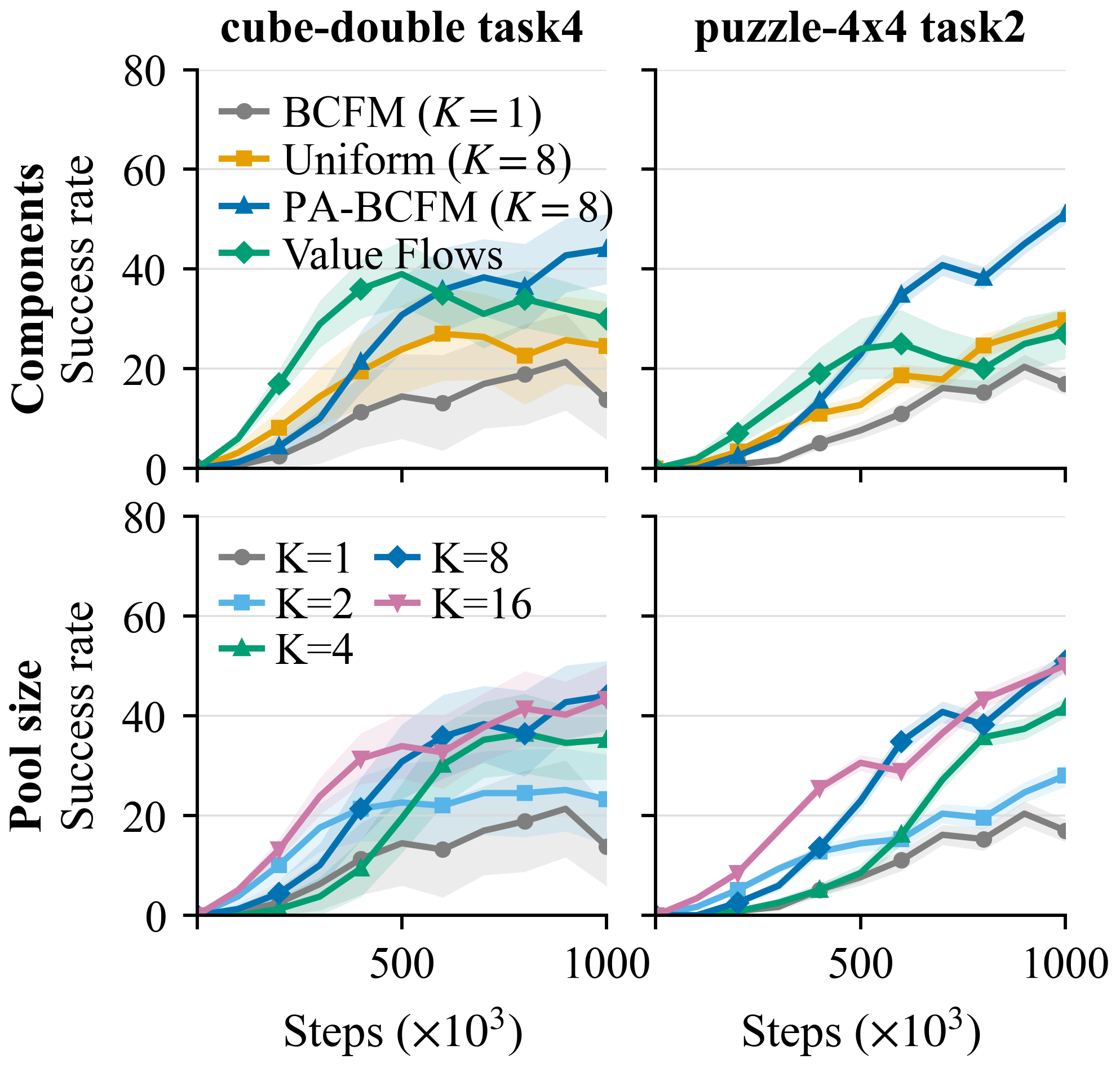}
    \caption{Estimator components and pool-size sensitivity.
    Top: BCFM ($K=1$), uniform and PA-BCFM ($K=8$), and Value Flows.
    Bottom: pool-size sensitivity on two state-based tasks.}
    \label{fig:estimator-ablation}
\end{figure}

\begin{figure}[!t]
    \centering
    \includegraphics[width=0.68\linewidth]
    {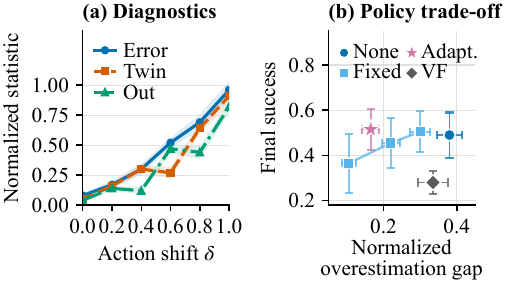}
    \caption{Critic diagnostics and adaptive pessimism on
    \texttt{puzzle-4x4-play-task4}.
    Left: Error, Twin, and Out denote absolute prediction error, twin
    disagreement, and unsupported mass. Right: None, Fixed, Adapt., and VF
    denote undistorted, fixed, adaptive, and Value Flows variants.}
    \label{fig:pool-diagnostics}
\end{figure}

\subsection{Offline and Online Performance}
\label{sec:offline-o2o-performance}

Table~\ref{tab:offline-eval-agg} summarizes offline performance. Fixed
policy-side components isolate critic learning among matched flow critics.
PA-BFM has the unique highest displayed mean on
eight of ten domains and ties for the highest mean on one more. Relative to
Value Flows, its state-based average rises from $51.2$ to $63.2$, its
image-based average from $33.4$ to $40.0$, and the ten-domain average from
$42.3$ to $51.6$. The task-level record is $35$ wins, $7$ ties, and $8$
losses across both observation modalities. Full task-level results appear in Appendix
Tables~\ref{tab:offline-eval-state-full}--\ref{tab:offline-eval-visual-full}
of the supplement.

After one million offline optimizer updates and one million online environment
steps, with an unchanged PA-BFM critic objective,
Figure~\ref{fig:o2o-aggregate} shows mean aggregate success of
$67.5$ at online entry and a final mean of $94.75$, versus $62.25$ and $86.25$
for Value Flows and $53.25$ and $93.5$ for FQL.

\subsection{Estimator and Pool-Size Sensitivity}
\label{sec:estimator-ablation}

Figure~\ref{fig:estimator-ablation} isolates posterior weighting under matched
$K=8$ endpoints and training; only PA-BCFM weights endpoints by bridge
compatibility. It improves final success by $19.5$ and $21.2$ points
over uniform aggregation on \texttt{cube-double task4} and
\texttt{puzzle-4x4 task2}, and by $14$ and $24$ points over Value Flows.

The pool-size sweep shows diminishing returns beyond $K=8$: increasing
$K$ from $1$ to $8$ raises final success from $13.8$ to $44$ and from $17$ to
$51$, whereas $K=16$ gives $43.4$ and $50.1$. On the representative state and
visual configurations, the reported $10^6$-update costs at $K=8$ are $4.0$
and $11.8$ hours, respectively, below the corresponding Value Flows times of
$4.4$ and $12.5$ hours. This
performance--cost saturation motivates the default $K=8$.

\subsection{Critic Diagnostics and Adaptive Pessimism}
\label{sec:critic-diagnostics}

On \texttt{puzzle-4x4-play-task4} at the final offline checkpoint, twin-critic
disagreement and unsupported mass track absolute prediction error under action
shift, with Spearman correlations $\rho_{\mathrm{twin}}=0.943$ and
$\rho_{\mathrm{out}}=0.886$. Fixed tilting reduces positive overestimation but
eventually degrades success; the adaptive rule reduces it from $0.380$ to
$0.165$ while raising final success from $0.490$ to $0.515$.
Appendix~\ref{app:pool-diagnostics} provides the evaluation protocol and
complete results for all evaluated fixed strengths.
\section{Conclusion}

We introduced PA-BFM, whose default target is an exact finite-pool
Rao--Blackwellization of the single-endpoint BCFM direction under the exact latent-anchor posterior.
It preserves the population velocity for every \(K\) and cannot increase
conditional target variance. We derive \(K^{-1}\)-scaling variance and,
under stated assumptions, optimization, terminal Wasserstein, and Bellman
tracking bounds. The endpoint pool enables conservative reweighting, while
the twin return flows provide critic-disagreement and unsupported-mass
diagnostics. Controlled and OGBench experiments show improved
estimation, return-law fidelity, and aggregate offline and offline-to-online
performance over Value Flows.

\begingroup
\small
\FloatBarrier
\bibliographystyle{plainnat}
\bibliography{references}
\endgroup

\clearpage
\setcounter{secnumdepth}{2}
\appendix
\section{Preliminaries}
\label{app:preliminaries}

\subsection{Core Notation and Bellman Endpoint Law}
\label{app:notation}

Fix a state--action pair and transition outcome
\(\tau=(r,s',d)\), and let \(\mu_\tau\) be the conditional Bellman endpoint
law. The estimator draws a pool whose twin-critic candidates are shared-noise
ensemble means:
\begin{align}
    \mathbf Y&=(Y_1,\ldots,Y_K),
    \qquad
    Y_j\overset{\mathrm{i.i.d.}}{\sim}\mu_\tau,
    \label{eq:app-abstract-endpoint-pool}\\
    Y_j^{(c)}
    &=r+\gamma(1-d)
      \phi_{\bar\theta_c}(\epsilon'_j\mid1,s',A'_j),
      \qquad c\in\{1,2\},\notag\\
    Y_j&=\tfrac12\bigl(Y_j^{(1)}+Y_j^{(2)}\bigr).
    \label{eq:app-conditional-endpoints}
\end{align}
\par\indent Here \(A'_j\sim\pi(\cdot\mid s')\),
\(\epsilon'_j\sim\mathcal N(0,1)\), and \(j=1,\ldots,K\). Candidates are
independent across \(j\), both target critics share
\((A'_j,\epsilon'_j)\), and \(\mu_\tau=\operatorname{Law}(Y_j\mid\tau)\).

\begin{table}[t]
    \centering
    \small
    \setlength{\tabcolsep}{2pt}
    \begin{tabular}{p{0.12\textwidth}p{0.34\textwidth}
                    p{0.12\textwidth}p{0.34\textwidth}}
        \toprule
        Symbol & Meaning & Symbol & Meaning \\
        \midrule
        \(\tau=(r,s',d)\) & transition outcome for fixed \((s,a)\)
        & \(\mu_\tau\) & conditional Bellman endpoint law \\
        \(Y\), \(\mathbf Y\) & endpoint and pool \((Y_1,\ldots,Y_K)\)
        & \(I\) & uniformly sampled anchor index \\
        \(X_t,x_t\) & bridge random variable and realization
        & \(\epsilon,p_0\) & \(\epsilon\sim\mathcal N(0,1)\) and its density \(p_0\) \\
        \(\ell_t(x\mid y)\) & bridge likelihood
        & \(u_t(x,y)\) & endpoint-conditioned bridge direction \\
        \(u_j\) & candidate direction \(u_t(x_t,Y_j)\)
        & \(w_j\) & posterior anchor responsibility \\
        \(\widehat u_K\) & posterior-mean training target
        & \(v_\theta\) & learned velocity field \\
        \(v_\tau^\star\) & transition-conditional population velocity
        & \(\ESS(\mathbf w)\) & \(1/\sum_jw_j^2\) \\
        \(K\) & ensemble-mean endpoint candidates in \(\mathbf Y\)
        & \(c\in\{1,2\}\) & critic index \\
        \(j\) & endpoint-candidate index, \(1\le j\le K\)
        & \(m\) & action-evaluation sample index, \(1\le m\le M_{\mathrm{act}}\) \\
        \bottomrule
    \end{tabular}
    \caption{Core notation.
    Uppercase variables denote random quantities and lowercase variables their
    realizations. Implementation- and proof-specific symbols are defined in
    the subsections where they are used.}
    \label{tab:appendix-notation}
\end{table}

\subsection{Distributional Bellman Operators}
\label{app:distributional-operators}

Let \(R\) and \(D\) denote reward and termination random variables. For a
fixed policy \(\pi\), the return satisfies
\begin{equation}
    \begin{aligned}
        Z^\pi(s,a)
        &\overset{d}{=}
        R+\gamma(1-D)Z^\pi(S',A'),\\[-2pt]
        A'&\sim\pi(\cdot\mid S').
    \end{aligned}
    \label{eq:app-distributional-bellman}
\end{equation}
\par\indent Under standard finite-moment assumptions, the induced operator is
a \(\gamma\)-contraction in maximal Wasserstein distance
~\citep{bellemare2017distributional}. For \(d=1\), the conditional endpoint
law is the point mass at \(r\).
\par\indent For an approximate iteration
\(Z_{n+1}=\widehat{\mathcal T}_nZ_n\), we write
\(\varepsilon_{K,n}=\bar W_1(\widehat{\mathcal T}_nZ_n,
\mathcal T^\pi Z_n)\);
Appendix~\ref{app:proof-td-propagation} propagates this error.

\subsection{One-Dimensional Return Flows}
\label{app:flow-matching}

A conditional return flow solves
\begin{equation}
    \begin{aligned}
        \frac{d}{dt}\phi_\theta(\epsilon\mid t,s,a)
        &=
        v_\theta\!\left(
            \phi_\theta(\epsilon\mid t,s,a),t,s,a
        \right),
        \\
        \phi_\theta(\epsilon\mid0,s,a)&=\epsilon.
    \end{aligned}
    \label{eq:app-flow-ode}
\end{equation}
\par\indent In one dimension, uniqueness of the ODE makes the terminal map
monotone and invertible on its image, enabling the inverse-flow CDF used in
Appendix~\ref{app:adaptive-distortion-contraction}.
\section{Theoretical Analysis}
\label{app:pa-bcfm-proofs}

This appendix proves the exact finite-pool Rao--Blackwell identity and its
variance, optimization, transport, and Bellman-tracking implications. Each
estimator analysis fixes the transition, target policy, and target critic, and
uses \(t\in[0,1-\delta]\) for \(\delta>0\).

\subsection{Path Geometry and Population Bridge Target}
\label{app:path-geometry}
\label{app:posterior-weights}

Fix \((s,a)\), let \((R,S',D)\) be the transition variables with realization
\(\tau=(r,s',d)\), and set \(X_0=\epsilon\sim\mathcal N(0,1)\). The endpoint
law \(\mu_\tau\) is defined in
Eq.~\eqref{eq:conditional-bellman-endpoint-law}. The results condition on this
law and hold for any fixed endpoint law satisfying the stated moments.

\paragraph{Slice-wise Bellman geometry.}
For a nonterminal transition (\(d=0\)), let the successor path be
$X_t'=(1-t)X_0+tX_1'$ with density
$p_t(\cdot\mid S',A')$ and velocity
$v_t(\cdot\mid S',A')$, define the slice-wise affine Bellman path
$X_t^{\mathrm{sl}}=r+\gamma X_t'$.
Conditional on $R=r$, the change of variables
$z=r+\gamma\xi$ gives the mixed density and velocity
\begin{align}
    q_t^{\mathrm{slice}}(z)
    =
    \frac{1}{\gamma}
    \E_{S',A'}
    \left[
        p_t(\xi\mid S',A')
    \right],
    \qquad
    \xi=\frac{z-r}{\gamma},
    \label{eq:app-slice-density}
    \\
    u_t^{\mathrm{slice}}(z)
    =
    \gamma
    \frac{
        \E\!
        \left[
            p_t(\xi\mid S',A')v_t(\xi\mid S',A')
        \right]
    }{
        \E\!
        \left[
            p_t(\xi\mid S',A')
        \right]
    }.
    \label{eq:app-slice-velocity}
\end{align}
\par\indent The second identity follows because each branch has
\(\dot X_t^{\mathrm{sl}}=\gamma\dot X_t'\), so its numerator is the mixed
probability flux.

At $t=0$, however,
$X_0^{\mathrm{sl}}=r+\gamma X_0\sim\mathcal N(r,\gamma^2)$, so the slice-wise path does
not satisfy the common Gaussian source used by the critic.
For a terminal transition, \(Y=r\) and \(\mu_\tau=\delta_r\).

\paragraph{Source-consistent endpoint bridge.}
The bridge used by BCFM and PA-BCFM instead draws
$Y=r+\gamma(1-d)X_1'$ and sets
\begin{equation}
    \begin{alignedat}{2}
    X_t&=(1-t)X_0+tY,\qquad & X_0&\perp Y,\\
    X_0&\sim\mathcal N(0,1), & Y&\sim\mu_\tau.
    \end{alignedat}
    \label{eq:app-source-consistent-bridge}
\end{equation}

Using the same source noise in the nonterminal successor path gives
\begin{align}
    X_t
    &=(1-t)X_0+t(R+\gamma X_1') \\
    &=R+\gamma X_t'
      +(1-t)\big[(1-\gamma)X_0-R\big],
    \label{eq:app-path-residual}
\end{align}
which proves Eq.~\eqref{eq:path-residual-correction}.
The residual vanishes at $t=1$, but not generally at intermediate times;
thus the two paths share the Bellman endpoint without sharing their
intermediate marginals or velocity fields.

Independence of $X_0$ and $Y\sim\mu_\tau$ gives
\begin{equation}
    p(X_t=x,Y\in dy\mid\tau)
    =
    \ell_t(x\mid y)\mu_\tau(dy).
\end{equation}

Bayes' rule therefore yields
\begin{equation}
    \begin{aligned}
    \nu_{\tau,x}(dy)
    &=
    \frac{\ell_t(x\mid y)\mu_\tau(dy)}{B_\tau(x,t)},
    \\
    v_\tau^\star(x,t)
    &=
    \int u_t(x,y)\nu_{\tau,x}(dy).
    \end{aligned}
    \label{eq:app-posterior-law}
\end{equation}

\begin{lemma}[Population target of BCFM]
\label{lem:bcfm-population-target}
Fix $(\tau,x,t)$ with $t<1$.
For the source-consistent bridge in
Eq.~\eqref{eq:app-source-consistent-bridge}, the conditional BCFM risk
\begin{equation}
    \mathcal R_\tau(u;x,t)
    =
    \E
    \left[
        \left(
            u-(Y-X_0)
        \right)^2
        \middle|
        X_t=x,\tau
    \right]
\end{equation}
has the unique $L_2$ minimizer
\begin{equation}
    u^\star
    =
    v_\tau^\star(x,t)
    =
    \E
    \left[
        \frac{Y-x}{1-t}
        \middle|
        X_t=x,\tau
    \right].
    \label{eq:app-bcfm-population-target}
\end{equation}
\end{lemma}

\begin{proof}
Since $X_t=(1-t)X_0+tY$, conditioning on $X_t=x$ gives
\begin{equation}
    Y-X_0
    =
    \frac{Y-x}{1-t}
    =
    u_t(x,Y).
\end{equation}

Let
$m_\tau(x,t)=\E[u_t(x,Y)\mid X_t=x,\tau]$.
The conditional squared risk decomposes as
\begin{align}
    \mathcal R_\tau(u;x,t)
    &=
    \E
    \left[
        \bigl(u-m_\tau+m_\tau-u_t(x,Y)\bigr)^2
        \middle|
        X_t=x,\tau
    \right] \\
    &=\bigl(u-m_\tau(x,t)\bigr)^2\\
    &\quad+\Var
    \left[
        u_t(x,Y)
        \middle|
        X_t=x,\tau
    \right].
\end{align}

The second term is independent of $u$, which proves the result.
\end{proof}

\paragraph{Transition mixture.}
Averaging \(v_\tau^\star(x,t)\) under the posterior transition weight
proportional to \(B_\tau(x,t)p(\tau\mid s,a)\) recovers the
source-consistent population bridge velocity. Endpoint pooling therefore
estimates the transition-conditional velocity, while replay sampling performs
the outer Bellman mixture.

\paragraph{Degeneracy of practical DCFM.}
\begin{remark}
The practical DCFM residual admits degenerate shared fields. For example,
\(v_\theta(z,t;s,a)=v_{\bar\theta}(z,t;s,a)=c(t)\) makes the residual vanish
independently of the reward and transition. Thus the residual alone does not
identify the Bellman transform; this observation does not assert that every
optimization trajectory converges to a constant field.
\end{remark}

\subsection{Finite-Pool Rao--Blackwellization}
\label{app:anchor-reuse}
\label{app:proof-pa-estimator}

Fix a transition outcome \(\tau\) and a flow time \(t<1\). Draw
\begin{equation}
    \begin{aligned}
    \mathbf Y&=(Y_1,\ldots,Y_K),
    \qquad Y_j\overset{\mathrm{i.i.d.}}{\sim}\mu_\tau,\\
    I&\sim\operatorname{Unif}\{1,\ldots,K\},
    \qquad X_t=(1-t)\epsilon+tY_I.
    \end{aligned}
    \label{eq:app-rb-sampling}
\end{equation}
where \(I\) is independent of the endpoint values and
\(\epsilon\sim\mathcal N(0,1)\). For a realized bridge point \(x\), define
\begin{equation}
    \begin{aligned}
    u_j(x,t)&=\frac{Y_j-x}{1-t},\qquad
    \omega_j(x,t)=\ell_t(x\mid Y_j),\\
    w_j(x,t)&=\frac{\omega_j(x,t)}
    {\sum_{q=1}^K \omega_q(x,t)}.
    \end{aligned}
    \label{eq:app-rb-components}
\end{equation}
\par\indent The single-endpoint BCFM target is
\(U_I=u_I(X_t,t)=Y_I-\epsilon\), whereas PA-BCFM uses
\begin{equation}
    \widehat u_K(X_t,t)=\sum_{j=1}^K w_j(X_t,t)u_j(X_t,t).
    \label{eq:app-rb-target}
\end{equation}

\paragraph{Posterior anchor index.}
Conditioned on \((X_t=x,\mathbf Y,\tau,t)\), the density of \(x\) under the
hypothesis \(I=j\) is \(\omega_j(x,t)\). Since the prior over \(I\) is uniform,
Bayes' rule gives
\begin{equation}
    \Pr(I=j\mid X_t=x,\mathbf Y,\tau,t)
    =w_j(x,t).
    \label{eq:app-anchor-posterior}
\end{equation}
\par\indent Consequently,
\begin{equation}
    \widehat u_K
    =
    \E\!\left[
        U_I
        \middle|
        X_t,\mathbf Y,\tau,t
    \right].
    \label{eq:app-rb-identity}
\end{equation}
\par\indent Taking the conditional expectation once more and using
Lemma~\ref{lem:bcfm-population-target} yields
\begin{equation}
    \begin{aligned}
    \E[\widehat u_K\mid X_t,\tau,t]
    &=\E[U_I\mid X_t,\tau,t]\\
    &=v_\tau^\star(X_t,t),
    \qquad \text{for every }K\ge1.
    \end{aligned}
    \label{eq:app-finite-k-unbiasedness}
\end{equation}
\par\indent At \(K=1\), \(\widehat u_1=U_I=Y_I-\epsilon\), proving exact reduction to
BCFM.

The outer transition mixture also gives
\(\E[\widehat u_K\mid X_t,s,a,t]=v^\star(X_t,t\mid s,a)\). Hence, for any
square-integrable prediction \(u(X_t,s,a,t)\),
\begin{align}
    \E\bigl[(u-\widehat u_K)^2\bigr]
    &=
    \E\bigl[(u-v^\star)^2\bigr]
    +
    \E\bigl[(\widehat u_K-v^\star)^2\bigr].
    \label{eq:app-rb-risk-decomposition}
\end{align}
\par\indent The second term does not depend on \(u\), so every finite pool preserves the
population minimizer.

\paragraph{Exact finite-pool variance reduction.}
Applying conditional total variance to
Eq.~\eqref{eq:app-rb-identity} gives
\begin{align}
    \Var(U_I\mid X_t,\tau,t)
    &=
    \Var(\widehat u_K\mid X_t,\tau,t)
    \nonumber\\
    &\quad+
    \E\!\left[
        \sum_{j=1}^K
        w_j(u_j-\widehat u_K)^2
        \middle|
        X_t,\tau,t
    \right].
    \label{eq:app-rb-total-variance}
\end{align}
\par\indent Therefore
\begin{equation}
    \Var(\widehat u_K\mid X_t,\tau,t)
    \le
    \Var(U_I\mid X_t,\tau,t)
    \label{eq:app-rb-variance-order}
\end{equation}
for every finite \(K\). The reduction is strict whenever the posterior over
anchor indices assigns positive mass to endpoints with different velocities.

\paragraph{Relation to leave-one-out SNIS.}
For asymptotic analysis, condition on the selected anchor and relabel it as
\(Y_0\). Then \(Y_0\sim\nu_{\tau,x}\), while
\(Y_1,\ldots,Y_{K-1}\) are i.i.d. from \(\mu_\tau\). Define
\begin{equation}
    \widetilde u_{K-1}
    =
    \frac{\sum_{j=1}^{K-1}\omega_ju_j}{\sum_{j=1}^{K-1}\omega_j}.
    \label{eq:app-loo-snis}
\end{equation}
\par\indent Direct algebra yields
\begin{equation}
\begin{aligned}
    \widehat u_K-\widetilde u_{K-1}
    &=
    \alpha_{0,K}
    \bigl(u_0-\widetilde u_{K-1}\bigr),\\
    \alpha_{0,K}
    &=
    \frac{\omega_0}{\sum_{j=0}^{K-1}\omega_j}.
\end{aligned}
\label{eq:app-aip-loo-decomposition}
\end{equation}
\par\indent Under a positive normalizer and the stated moment conditions,
\(\alpha_{0,K}=O_p(K^{-1})\), so
\begin{equation}
    \widehat u_K-\widetilde u_{K-1}=O_p(K^{-1}).
    \label{eq:app-aip-loo-rate}
\end{equation}
\par\indent The practical estimator is therefore consistent and has the same first-order
asymptotic distribution as standard SNIS. In finite samples, however,
Eq.~\eqref{eq:app-finite-k-unbiasedness} is the sharper statement: the anchor
correction exactly cancels the ratio bias in expectation under the uniform
anchor-sampling protocol.

\subsection{Asymptotic Variance, ESS, and Gradient Covariance}
\label{app:asymptotic-ess}
\label{app:proof-mean-equivalence}
\label{app:proof-ess}

\paragraph{Ratio asymptotics.}
Fix \((x,t,\tau)\), and write
\(\omega(Y)=\ell_t(x\mid Y)\),
\(u(Y)=u_t(x,Y)\),
\(B_\tau(x,t)=\E_{\mu_\tau}[\omega(Y)]\), and
\(v_\tau^\star=\E[\omega(Y)u(Y)]/B_\tau(x,t)\).
The multivariate central limit theorem for
\((\omega(Y)u(Y),\omega(Y))\), followed by the delta method, gives influence variable
\begin{equation}
    \psi(Y)
    =
    \frac{\omega(Y)(u(Y)-v_\tau^\star)}{B_\tau(x,t)}
\end{equation}
and asymptotic variance
\begin{equation}
    \sigma_\tau^2(x,t)
    =
    \frac{
        \E_{\mu_\tau}
        \left[
            \omega(Y)^2(u(Y)-v_\tau^\star)^2
        \right]
    }{\left(\E_{\mu_\tau}[\omega(Y)]\right)^2}.
    \label{eq:app-rb-asymptotic-variance}
\end{equation}
\par\indent Because the anchor correction in
Eq.~\eqref{eq:app-aip-loo-rate} is \(o_p(K^{-1/2})\), Slutsky's theorem gives
\begin{equation}
    \begin{aligned}
    \sqrt K(\widehat u_K-v_\tau^\star)
    &\Rightarrow \mathcal N(0,\sigma_\tau^2(x,t)),\\
    \Var(\widehat u_K\mid x,t,\tau)
    &=\frac{\sigma_\tau^2(x,t)}{K}+o(K^{-1}).
    \end{aligned}
    \label{eq:app-rb-clt}
\end{equation}

\paragraph{Effective sample size.}
The realized effective sample size is
\begin{equation}
    \ESS(\mathbf w)
    =
    \frac{(\sum_j \omega_j)^2}{\sum_j \omega_j^2}
    =
    \frac{1}{\sum_jw_j^2}.
    \label{eq:app-ess-definition}
\end{equation}
\par\indent The anchor contributes only one term, so the strong law applied to the
nonanchor endpoints yields
\begin{equation}
    \frac{\ESS(\mathbf w)}{K}
    \xrightarrow{a.s.}
    \chi_\tau(x,t)
    =
    \frac{\left(\E_{\mu_\tau}[\omega(Y)]\right)^2}
         {\E_{\mu_\tau}[\omega(Y)^2]}.
    \label{eq:app-ess-limit}
\end{equation}
\par\indent If \(|u(Y)-v_\tau^\star|\le C_{u,\tau}(x,t)\), then
\begin{equation}
    \begin{aligned}
    \sigma_\tau^2(x,t)
    &\le \frac{C_{u,\tau}(x,t)^2}{\chi_\tau(x,t)},\\
    \Var(\widehat u_K\mid x,t,\tau)
    &\le \frac{C_{u,\tau}(x,t)^2}{K\chi_\tau(x,t)}+o(K^{-1}).
    \end{aligned}
    \label{eq:app-ess-variance-bound}
\end{equation}
\par\indent Thus ESS measures posterior-weight concentration and controls an asymptotic
variance bound; it is not an unconditional exact variance identity.

\paragraph{Gradient Rao--Blackwellization.}
Fix \((X_t,\tau,t,\theta)\), and let
\(J_\theta=\nabla_\theta v_\theta(X_t,t;s,a)\). The single-endpoint BCFM
gradient and PA-BCFM gradient are
\begin{equation}
    g_1=2J_\theta\bigl(v_\theta-U_I\bigr),
    \qquad
    g_K=2J_\theta\bigl(v_\theta-\widehat u_K\bigr).
    \label{eq:app-gradient-definitions}
\end{equation}
\par\indent Equation~\eqref{eq:app-rb-identity} implies
\begin{equation}
    \begin{aligned}
    g_K
    &=\E[g_1\mid X_t,\mathbf Y,\tau,t,\theta],\\
    \Cov(g_K\mid X_t,\tau,t,\theta)
    &\preceq \Cov(g_1\mid X_t,\tau,t,\theta).
    \end{aligned}
    \label{eq:app-gradient-rb}
\end{equation}
\par\indent Moreover,
\begin{equation}
    \Cov(g_K\mid x,t,\tau,\theta)
    =
    \frac{4\sigma_\tau^2(x,t)}{K}
    J_\theta J_\theta^\top
    +o(K^{-1}).
    \label{eq:app-gradient-asymptotic-covariance}
\end{equation}
\par\indent This result requires no conditional independence assumption among candidate
gradients.

\subsection{Optimization, Transport, and Bellman Tracking}
\label{app:control-variate-propagation}
\label{app:flow-stability}
\label{app:proof-td-propagation}
\label{app:pcbf-comparison}

\paragraph{Relation to PCBF.}
PCBF adds a learned successor-field control variate with coefficient
\(\lambda_{\mathrm{cv}}\), so its population field differs from the
source-consistent target by approximation and cross-conditioning terms
\citep{xu2026pcbf}. PA-BCFM instead preserves the exact population target for
every finite \(K\); its pool-dependent error is variance rather than target
bias.

\paragraph{Optimization effect of the endpoint pool.}
Define the population field risk
\begin{equation}
    \begin{aligned}
    \mathcal F(\theta)
    &=\frac12\E\!\left[
        |v_\theta(X_t,t;s,a)-v^\star(X_t,t;s,a)|^2
    \right],\\
    \mathcal F_\star&=\inf_\theta\mathcal F(\theta).
    \end{aligned}
    \label{eq:app-field-risk}
\end{equation}
\par\indent Exact finite-pool unbiasedness implies that the PA-BCFM stochastic gradient is
unbiased for \(\nabla\mathcal F\). Under uniform integrability of the
conditional variance bound in Appendix~\ref{app:asymptotic-ess}, total
variance yields constants \(\sigma_{\mathrm{base}}^2\) and
\(C_{\mathrm{pool}}\) such that a minibatch of size \(B_{\mathrm{mb}}\) satisfies
\begin{equation}
    \E\!\left[
        \|\widehat g_{K,B_{\mathrm{mb}}}-\nabla\mathcal F(\theta)\|^2
    \right]
    \le
    \frac1{B_{\mathrm{mb}}}
    \left(
        \sigma_{\mathrm{base}}^2
        +
        \frac{C_{\mathrm{pool}}}{K}
    \right).
    \label{eq:app-k-gradient-variance}
\end{equation}
\par\indent The first term collects transition, state--action, and flow-time randomness;
the second is the endpoint-pool contribution.

Assume \(\mathcal F\) is \(L_{\mathcal F}\)-smooth and satisfies the
Polyak--\L{}ojasiewicz condition with constant \(\mu_{\mathrm{PL}}\), using
\(\tfrac12\|\nabla\mathcal F(\theta)\|_2^2
\ge\mu_{\mathrm{PL}}(\mathcal F(\theta)-\mathcal F_\star)\). For
\(\theta_{n+1}=\theta_n-\eta_{\mathrm{opt}}\widehat g_{K,B_{\mathrm{mb}}}\) with \(\eta_{\mathrm{opt}}\le1/L_{\mathcal F}\),
smoothness and unbiasedness give
\begin{align}
    \E[\mathcal F(\theta_{n+1})-\mathcal F_\star\mid\theta_n]
    &\le
    (1-\mu_{\mathrm{PL}}\eta_{\mathrm{opt}})
    (\mathcal F(\theta_n)-\mathcal F_\star)
    \nonumber\\
    &\quad+
    \frac{L_{\mathcal F}\eta_{\mathrm{opt}}^2}{2B_{\mathrm{mb}}}
    \left(
        \sigma_{\mathrm{base}}^2
        +
        \frac{C_{\mathrm{pool}}}{K}
    \right).
\end{align}
\par\indent Unrolling the recursion proves
\begin{align}
    \E[\mathcal F(\theta_T)-\mathcal F_\star]
    &\le
    (1-\mu_{\mathrm{PL}}\eta_{\mathrm{opt}})^T
    \bigl(\mathcal F(\theta_0)-\mathcal F_\star\bigr)
    \nonumber\\
    &\quad+
    \frac{L_{\mathcal F}\eta_{\mathrm{opt}}}{2\mu_{\mathrm{PL}}B_{\mathrm{mb}}}
    \left(
        \sigma_{\mathrm{base}}^2
        +
        \frac{C_{\mathrm{pool}}}{K}
    \right),
    \label{eq:app-k-sgd-bound}
\end{align}
which gives the optimization statement in
Corollary~\ref{cor:pa-training-transport}.

\paragraph{Velocity-to-endpoint stability.}
Couple the ideal and learned ODEs by their common source and let
\(e_t=\widehat X_t-X_t\). If the learned field is \(L_v\)-Lipschitz in its
state argument, then Gr\"onwall's inequality gives
\begin{equation}
    |e_1|
    \le
    \int_0^1
    e^{L_v(1-t)}
    |v_{\theta_T}(X_t,t)-v^\star(X_t,t)|dt.
\end{equation}
\par\indent Let \(Z_{\theta_T,1}^{\mathrm{exact}}\) denote the terminal law
obtained by exact integration of the learned ODE and
\(Z_{\theta_T,1}^{(L)}\) the law produced by the \(L\)-step numerical solver.
We write
\begin{equation}
    \delta_{\mathrm{ode}}
    =
    \E\!\left[
        W_1\!\left(
            Z_{\theta_T,1}^{(L)},
            Z_{\theta_T,1}^{\mathrm{exact}}
        \right)
    \right]
    \label{eq:app-ode-integration-error}
\end{equation}
for the expected numerical-integration contribution.

\par\indent Cauchy--Schwarz, the common-source coupling, Jensen's inequality,
and the triangle inequality imply
\begin{equation}
    \begin{aligned}
    \E W_1(Z_{\theta_T,1}^{(L)},Z_1^\star)
    &\le C_{L_v}\sqrt{2\E\mathcal F(\theta_T)}
    +\delta_{\mathrm{ode}},\\
    C_{L_v}^2
    &=\int_0^1e^{2L_v(1-t)}dt.
    \end{aligned}
    \label{eq:app-velocity-w1}
\end{equation}
\par\indent Substituting Eq.~\eqref{eq:app-k-sgd-bound} into
Eq.~\eqref{eq:app-velocity-w1} proves
the transport statement in Corollary~\ref{cor:pa-training-transport}. Thus the endpoint-pool contribution is
\(O(K^{-1})\) in the optimization floor and \(O(K^{-1/2})\) in the terminal
Wasserstein bound.

\paragraph{Bellman tracking.}
Let \(\varepsilon_{K,n}\) denote the one-step terminal-distribution error,
including Eq.~\eqref{eq:app-velocity-w1}, numerical integration error, and any
target-law mismatch. Distributional Bellman contraction gives, for fixed
policy,
\begin{equation}
    d_{n+1}
    \le
    \gamma d_n+\varepsilon_{K,n},
    \qquad
    \limsup_n d_n
    \le
    \frac{\limsup_n\varepsilon_{K,n}}{1-\gamma}.
    \label{eq:app-fixed-policy-tracking}
\end{equation}
\par\indent For a changing policy, one additional triangle inequality yields
\begin{equation}
    d_{n+1}
    \le
    \gamma d_n+\varepsilon_{K,n}+\Delta_n^\pi,
    \label{eq:app-policy-drift}
\end{equation}
where \(\Delta_n^\pi\) measures the change between consecutive policy return
laws. This is a tracking result, not a global actor--critic convergence claim.
\subsection{Conservative Endpoint Reweighting}
\label{app:distorted-laws}
\label{app:endpoint-pool-interface-proofs}
\label{app:distorted-contraction}
\label{app:empirical-distorted-consistency}

Let \(g:\mathbb R\to[0,\infty)\) be a fixed distortion with
\(0<\int g(y)\mu_\tau(dy)<\infty\). The distorted endpoint law is
\begin{equation}
    \mu_{\tau,g}(dy)
    =
    \frac{g(y)\mu_\tau(dy)}
         {\int g(\widetilde y)\mu_\tau(d\widetilde y)}.
    \label{eq:app-distorted-endpoint-law}
\end{equation}

After observing \(X_t=x\), its bridge posterior is proportional to
\(\ell_g(x\mid y)\mu_\tau(dy)\), where
\begin{equation}
    \ell_g(x\mid y)=\ell_t(x\mid y)g(y).
    \label{eq:app-distorted-likelihood}
\end{equation}

For \(g\equiv1\), the weights coincide with the posterior anchor
probabilities in Eq.~\eqref{eq:app-anchor-posterior}, giving the exact
finite-pool Rao--Blackwell identity. For nonconstant \(g\), the weights
intentionally define a different target and are not the posterior of the
uniformly selected anchor.

The practical distorted weights and target are
\begin{equation}
    w_j^g
    =
    \frac{\ell_t(x\mid Y_j)g(Y_j)}
         {\sum_{\ell=1}^{K}\ell_t(x\mid Y_\ell)g(Y_\ell)},
    \qquad
    \widehat u_K^g
    =
    \sum_{j=1}^{K}w_j^g u_t(x,Y_j).
    \label{eq:app-distorted-aip-estimator}
\end{equation}

\begin{proposition}[Consistency of the distorted endpoint estimator]
\label{prop:distorted-aip-consistency}
Under the moment and positive-normalizer assumptions of
Theorem~\ref{thm:pa-bcfm-estimator}, condition on the selected anchor and
relabel it as \(Y_0\). With unnormalized weight \(\ell_g\),
\begin{equation}
    \begin{aligned}
    \widehat u_K^g
    -\widetilde u_{K-1}^g
    &=\alpha_{0,K}^{g}
    \bigl(u_0-\widetilde u_{K-1}^g\bigr),
    \\
    \alpha_{0,K}^{g}
    &=\frac{\ell_g(x\mid Y_0)}{\sum_{j=0}^{K-1}\ell_g(x\mid Y_j)}.
    \end{aligned}
    \label{eq:app-distorted-aip-loo}
\end{equation}

Consequently,
\(\widehat u_K^g-\widetilde u_{K-1}^g=O_p(K^{-1})\),
\begin{equation}
    \widehat u_K^g
    \xrightarrow{p}
    v_{\tau,g}^\star(x,t)
    =
    \frac{\E_{\mu_\tau}[\ell_g(x\mid Y)u_t(x,Y)]}
         {\E_{\mu_\tau}[\ell_g(x\mid Y)]},
    \label{eq:app-distorted-population-velocity}
\end{equation}
and the practical estimator has the same first-order SNIS variance as the
leave-one-out ratio. For nonconstant \(g\), its finite-pool ratio is generally
biased by \(O(K^{-1})\); the exact unbiasedness in
Eq.~\eqref{eq:app-finite-k-unbiasedness} is specific to the undistorted
objective.
\end{proposition}

\begin{proof}
Equation~\eqref{eq:app-distorted-aip-loo} follows from the same ratio identity
as Eq.~\eqref{eq:app-aip-loo-decomposition}. The nonanchor samples remain
independent draws from \(\mu_\tau\), so the strong law, ratio expansion, and
central limit theorem apply to \(\ell_g\). The anchor contributes one additional
term and changes the leave-one-out ratio by only \(O_p(K^{-1})\). Unlike the
case \(g\equiv1\), no tower-property cancellation makes this distorted ratio
exactly unbiased at finite \(K\).
\end{proof}

\paragraph{Hard truncation.}
For a population truncation level \(\rho\in[0,1)\), let
\(q_{1-\rho}\) be the \((1-\rho)\)-quantile of \(\mu_\tau\) and define
\(g_\rho^{\mathrm{hard}}(y)=\mathbf 1\{y\le q_{1-\rho}\}\).
Replacing the population quantile by the empirical pool quantile is
consistent under continuity of the endpoint CDF at the truncation point. The
soft tilt \(g_\eta(y)=\exp(-\eta y)\) is covered directly by
Proposition~\ref{prop:distorted-aip-consistency}.

\begin{proposition}[Distorted Bellman contraction and pessimistic order]
\label{prop:distorted-bellman-order}
Let \(\mathcal D_g\) be the population distortion and suppose
\begin{equation}
    W_p(\mathcal D_g\mu,\mathcal D_g\nu)
    \le L_gW_p(\mu,\nu).
    \label{eq:app-distortion-lipschitz}
\end{equation}

Then
\begin{equation}
    \bar W_p(\mathcal T_g^\pi Z_1,\mathcal T_g^\pi Z_2)
    \le
    \gamma L_g\bar W_p(Z_1,Z_2).
    \label{eq:app-distorted-bellman-contraction}
\end{equation}

Hence \(\gamma L_g<1\) gives a unique fixed point. If \(g\) is
nonincreasing, then \(\mathcal D_g\mu\preceq_{\mathrm{FSD}}\mu\). If the
Bellman operator is monotone and both distorted and undistorted iterations
converge, their fixed points satisfy
\(Z_g^\pi\preceq_{\mathrm{FSD}}Z^\pi\).
\end{proposition}

\begin{proof}
The Wasserstein bound follows from Bellman affine scaling and the assumed
Lipschitz property of \(\mathcal D_g\). For nonincreasing \(g\), the standard
covariance inequality implies first-order stochastic dominance. Bellman
monotonicity then preserves this order along both convergent iterations.
\end{proof}

\subsection{Adaptive Distortion and Reliability Diagnostics}
\label{app:adaptive-distortion-contraction}

\paragraph{Scope of the adaptive schedule.}
Proposition~\ref{prop:distorted-bellman-order} applies to a fixed population
distortion \(g\). The practical adaptive rule instead uses
\(D_{\mathrm{pool}}\), which depends on the training transition, bridge point,
and sampled endpoint pool, and is treated with stop-gradient within each
critic update. Freezing this realized coefficient for differentiation does
not define a common distorted Bellman operator across endpoint laws or
training iterations. We therefore use the adaptive schedule only as a
practical extension and make no operator-level contraction claim for it.
This training-time quantity is distinct from the state--action diagnostic
\(D_{\mathrm{twin}}(s,a)\).

\paragraph{Inverse-flow support probability.}
For a one-dimensional ODE with unique solutions, each target terminal map
\(\epsilon\mapsto\phi_{\bar\theta_c}(\epsilon\mid1,s,a)\) is order
preserving. For \(c\in\{1,2\}\), define
\begin{align}
    \widehat F_c(z\mid s,a)
    &=\Phi\!\left(
        \phi_{\bar\theta_c}^{-1}(z\mid1,s,a)
    \right),
    \label{eq:app-inverse-flow-cdf}\\
    \widehat F_{\mathrm{ens}}(z\mid s,a)
    &=\frac12\sum_{c=1}^{2}\widehat F_c(z\mid s,a).
    \nonumber
\end{align}
\par\indent For the fixed empirical upper-return threshold
\(q_{\mathrm{hi}}^{\mathcal D}\), the reported unsupported mass is
\begin{equation}
    P_{\mathrm{out}}(s,a)
    =1-\widehat F_{\mathrm{ens}}
      (q_{\mathrm{hi}}^{\mathcal D}\mid s,a).
    \label{eq:app-support-tail-probability}
\end{equation}

This is a model-implied tail probability rather than an OOD certificate.

\paragraph{Statistical and dynamical conditions.}
All finite-pool, asymptotic, optimization, and ODE conclusions hold under the
assumptions stated in their respective propositions; none imply global
actor--critic convergence.
\section{Practical Algorithm}
\label{app:practical-algorithm}

This appendix gives the full critic update, the optional endpoint-pool
extensions, and the policy-improvement procedure used in the experiments.

\subsection{PA-BCFM Critic Update}
\label{app:critic-update}

\paragraph{Endpoint construction.}
For each transition \((s,a,r,s',d)\) in a minibatch of size \(B_{\mathrm{mb}}\), write
\(\tau=(r,s',d)\) and draw
\begin{equation}
    A'_j\sim\pi(\cdot\mid s'),
    \qquad
    \epsilon'_j\sim\mathcal N(0,1),
    \qquad j=1,\ldots,K,
\end{equation}
and generate common Bellman endpoints from the ensemble-mean target return
flow,
\begin{equation}
    \begin{aligned}
    Y_j^{(c)}
    &=r+\gamma(1-d)
      \phi_{\bar\theta_c}(\epsilon'_j\mid1,s',A'_j),\\
    Y_j
    &=\frac12\left(Y_j^{(1)}+Y_j^{(2)}\right),
      \qquad c\in\{1,2\}.
    \end{aligned}
    \label{eq:app-endpoint-construction}
\end{equation}

The two target critics share successor actions and return-noise samples. Their
endpoint outputs are averaged into the common endpoint \(Y_j\), and the resulting
PA-BCFM target is used for both critic losses.

Both target flows are evaluated vectorially on the \(K\) shared samples and
reuse the successor-state encoding. We sample
\(I\sim\operatorname{Unif}\{1,\ldots,K\}\), independently of endpoint
values, then draw \(\epsilon\sim\mathcal N(0,1)\) and
\(t\sim\mathcal U[0,1-\epsilon_t]\). The bridge point is
\begin{equation}
    x_t=(1-t)\epsilon+tY_I.
    \label{eq:app-anchor-bridge}
\end{equation}

\paragraph{Posterior target.}
For a distortion \(g\), define the log weight
\begin{equation}
    \Lambda_j
    =
    -\frac{(x_t-tY_j)^2}{2(1-t)^2}
    +\log g(Y_j).
    \label{eq:app-stable-logit}
\end{equation}

The default PA-BCFM objective uses \(g\equiv1\). We normalize with
log-sum-exp,
\begin{equation}
    w_j
    =
    \exp\!\left(
        \Lambda_j-\operatorname{logsumexp}_{\ell=1}^{K}\Lambda_\ell
    \right),
    \label{eq:app-stable-weight}
\end{equation}
and construct the candidate directions and posterior target
\begin{equation}
    u_j=\frac{Y_j-x_t}{1-t},
    \qquad
    \widehat u_K=\sum_{j=1}^{K}w_ju_j.
    \label{eq:app-posterior-target}
\end{equation}

The endpoints, weights, and posterior target are stop-gradient quantities.

\paragraph{Critic and target updates.}
Each member of the twin critic minimizes
\begin{equation}
    \mathcal L_c
    =
    \E\left[
        \left\|
        v_{\theta_c}(x_t,t;s,a)
        -\operatorname{sg}[\widehat u_K]
        \right\|_2^2
    \right],
    \qquad c\in\{1,2\}.
    \label{eq:app-twin-critic-loss}
\end{equation}

The two online critics are parameterized and optimized separately despite
sharing this target; their individual target flows are retained for the
training-time pool signal and direct state--action diagnostics in
Appendix~\ref{app:pool-control}.

After the gradient update, the corresponding target parameters are updated by
Polyak averaging,
\begin{equation}
    \bar\theta_c
    \leftarrow
    \tau_{\mathrm{tgt}}\theta_c
    +(1-\tau_{\mathrm{tgt}})\bar\theta_c.
    \label{eq:app-polyak-update}
\end{equation}

Gradients do not pass through the target flow, posterior construction,
diagnostics, or target update.

\begin{algorithm}[t]
\caption{Posterior-Averaged Bellman Flow Matching (PA-BFM)}
\label{alg:pa-bfm}
\begin{algorithmic}[1]
\REQUIRE replay buffer \(\mathcal D\), twin critics and targets, policy
\(\pi\), behavior proposal \(\pi_\beta\), endpoint count \(K\), optional
endpoint distortion \(g\), policy-training schedule \(\mathcal P\)
\REPEAT
    \STATE Sample \(B_{\mathrm{mb}}\) transitions \((s,a,r,s',d)\) from \(\mathcal D\)
    \STATE Draw \(K\) successor actions and target-flow noises per transition
    \STATE Evaluate both targets and set \(Y_j=(Y_j^{(1)}+Y_j^{(2)})/2\)
    \STATE Sample \(I\) uniformly and construct \(x_t=(1-t)\epsilon+tY_I\)
    \STATE Form stable posterior logits, weights, and posterior target \(\widehat u_K\)
    \STATE Update both critics using Eq.~\eqref{eq:app-twin-critic-loss}
    \STATE Polyak-update both target critics
    \IF{online interaction is active}
        \STATE Collect transitions and append them to \(\mathcal D\)
    \ENDIF
    \IF{one-step policy training is enabled by \(\mathcal P\)}
        \STATE Update the behavior-regularized one-step policy with
        Eq.~\eqref{eq:app-online-policy-loss}
    \ENDIF
\UNTIL{the training budget is exhausted}
\RETURN twin return critics, behavior proposal, and optional one-step policy
\end{algorithmic}
\end{algorithm}

\subsection{Conservative Reweighting and Diagnostics}
\label{app:pool-control}

\paragraph{Fixed endpoint distortions.}
The main PA-BCFM experiments use the undistorted target \(g\equiv1\). The
reported conservative variant uses the fixed soft tilt
\begin{equation}
    g_\eta(Y_j)=\exp(-\eta Y_j),
    \qquad \eta\ge0.
\end{equation}
\par\indent Hard upper-tail truncation is implemented by zeroing candidates
above the empirical pool quantile; it is not used in the reported
experiments. The exact finite-pool Rao--Blackwell identity applies to
\(g\equiv1\); nonconstant \(g\) deliberately defines a distorted target.

\paragraph{Training-time pool disagreement.}
To separate critic disagreement from candidate-sampling noise, both target
critics are evaluated on the same successor actions and return noises:
\begin{equation}
    Y_j^{(c)}
    =
    r+\gamma(1-d)
    \phi_{\bar\theta_c}(\epsilon'_j\mid1,s',A'_j),
    \qquad c\in\{1,2\}.
    \label{eq:app-twin-endpoints}
\end{equation}

Let \(Y_j=(Y_j^{(1)}+Y_j^{(2)})/2\), and compute the undistorted
posterior weights \(w_j^0\) from \(Y_j\). The disagreement score is
\begin{equation}
    D_{\mathrm{pool}}(\tau,x_t,t,\mathbf Y)
    =
    \sum_{j=1}^{K}w_j^0
    \left(Y_j^{(1)}-Y_j^{(2)}\right)^2.
    \label{eq:app-twin-disagreement}
\end{equation}

\paragraph{Adaptive distortion.}
For the reported diagnostic experiment, return-valued quantities are scaled
by the task-specific
\(\Delta_e=q_{.95}^{\mathcal D,e}-q_{.05}^{\mathcal D,e}\)
return-to-go range, and normalized disagreement is clipped to \([0,1]\).
Training-time pool disagreement modulates the soft tilt as
\begin{equation}
    \begin{aligned}
    \eta(\tau,x_t,t)
    &=\eta_0\operatorname{sg}\!\left[
        \operatorname{clip}\!\left(
            \frac{\sqrt{D_{\mathrm{pool}}(\tau,x_t,t,\mathbf Y)}}{\Delta_e},
            0,1
        \right)
    \right],
    \\
    g_{\tau,x_t,t}(Y_j)
    &=\exp\!\left[
        -\eta(\tau,x_t,t)\frac{Y_j}{\Delta_e}
    \right].
    \end{aligned}
    \label{eq:app-adaptive-distortion}
\end{equation}

The stop-gradient makes the diagnostic a target-construction signal rather
than a second gradient path through the critic. This sample-dependent rule is
not covered by the fixed-distortion contraction result.

\paragraph{State--action reliability diagnostics.}
At a queried \((s,a)\), shared Gaussian samples isolate disagreement between
the twin target return flows:
\begin{equation}
    \begin{aligned}
    Z_m^{(c)}(s,a)
    &=\phi_{\bar\theta_c}(\xi_m\mid1,s,a),
      \qquad c\in\{1,2\},\\
    \xi_m&\sim\mathcal N(0,1),
      \qquad m=1,\ldots,M_{\mathrm{diag}},\\
    D_{\mathrm{twin}}(s,a)
    &=\frac1{M_{\mathrm{diag}}}
      \sum_{m=1}^{M_{\mathrm{diag}}}
      \bigl(Z_m^{(1)}(s,a)-Z_m^{(2)}(s,a)\bigr)^2.
    \end{aligned}
    \label{eq:app-state-action-disagreement}
\end{equation}
\par\indent For each one-dimensional return flow, inversion gives the critic CDF; the
reported unsupported mass uses their ensemble average:
\begin{align}
    \widehat F_c(b\mid s,a)
    &=\Phi\!\left(
        \phi_{\bar\theta_c}^{-1}(b\mid1,s,a)
    \right),
    \label{eq:app-support-score}\\
    \widehat F_{\mathrm{ens}}(b\mid s,a)
    &=\frac12\sum_{c=1}^{2}\widehat F_c(b\mid s,a),
    \nonumber\\
    P_{\mathrm{out}}(s,a)
    &=1-\widehat F_{\mathrm{ens}}
      (q_{\mathrm{hi}}^{\mathcal D}\mid s,a).
    \nonumber
\end{align}
\par\indent The threshold \(q_{\mathrm{hi}}^{\mathcal D}\) is a
dataset-derived upper-return threshold selected once before diagnostic
evaluation and then held fixed.

\subsection{Policy Extraction and Shared Policy Training}
\label{app:policy-extraction}

\par\indent All main experiments rank candidate actions by predicted mean
return. Where policy training is enabled by \(\mathcal P\), actions are
distilled into a one-step policy using the shared behavior-regularized
objective below. This policy-side procedure is matched across PA-BFM, Value
Flows, and PCBF.

\paragraph{Distributional action evaluation.}
For a state \(s\), the behavior flow proposes \(N\) actions
\(a_i\sim\pi_\beta(\cdot\mid s)\). Each action is evaluated with
\(M_{\mathrm{act}}\) shared return noises under the twin return flows:
\begin{equation}
    \begin{aligned}
    z_{im}^{(c)}
    &=\phi_{\theta_c}(\epsilon_{im}\mid1,s,a_i),\\
    \epsilon_{im}&\sim\mathcal N(0,1),\qquad
    \substack{m=1,\ldots,M_{\mathrm{act}},\\c\in\{1,2\}}.
    \end{aligned}
    \label{eq:app-policy-return-samples}
\end{equation}

The mean score and selected action are
\begin{equation}
    \begin{aligned}
    \widehat\mu_i
    &=\frac{1}{2M_{\mathrm{act}}}
      \sum_{c=1}^{2}\sum_{m=1}^{M_{\mathrm{act}}}z_{im}^{(c)},\\
    a^\star
    &=\operatorname*{arg\,max}_{1\le i\le N}\widehat\mu_i.
    \end{aligned}
    \label{eq:app-action-selection}
\end{equation}

\par\indent With the default \(M_{\mathrm{act}}=1\) in
Table~\ref{tab:common-hyperparameters}, \(\widehat\mu_i\) is a single-sample
Monte Carlo estimate of the critic mean.

\paragraph{Shared one-step policy update.}
When policy training is enabled by \(\mathcal P\), a one-step policy
\(\pi_\psi\), with parameters \(\psi\), is trained against the frozen critic
with behavior regularization:
\begin{equation}
    \begin{aligned}
    a_\psi&=\pi_\psi(s,\zeta),\\
    \mathcal L_{\mathrm{policy}}(\psi)
    &=\E_{\substack{s\sim\mathcal D\\\zeta\sim\mathcal N(0,\mathbf I_{d_a})}}\Bigl[
        -\widehat\mu(s,a_\psi)\\[-2pt]
    &\qquad
        +\alpha_{\mathrm{BC}}
        \left\|
            a_\psi-\pi_\beta(s,\zeta)
        \right\|_2^2
    \Bigr].
    \end{aligned}
    \label{eq:app-online-policy-loss}
\end{equation}

Here \(\widehat\mu(s,a)\) is the twin-flow sample mean above evaluated at
action \(a\), \(\zeta\) is action-generation noise,
\(\mathbf I_{d_a}\) is the identity covariance in action dimension, and
\(\alpha_{\mathrm{BC}}\ge0\) controls behavior regularization. Critic
parameters are fixed during this actor update.

\subsection{Implementation and Complexity}
\label{app:complete-algorithm}

\paragraph{Numerical implementation.}
We restrict flow time to \(t\le1-\epsilon_t\), compute posterior weights with
log-sum-exp, and reuse the successor-state representation across all \(K\)
endpoints. Target-flow solves are vectorized and run without gradient storage.
The same ODE solver and precision policy are used across estimator variants.

\paragraph{Complexity.}
With \(T_{\mathrm{flow}}\) target-flow solver steps, endpoint construction
costs \(O(B_{\mathrm{mb}}KT_{\mathrm{flow}})\) network work and
\(O(B_{\mathrm{mb}}K)\) posterior arithmetic and storage. Policy extraction
requires \(O(NM_{\mathrm{act}}T_{\mathrm{flow}})\) return-flow evaluations per
state. End-to-end wall time is reported separately from these asymptotic
counts.
\section{Additional Experiments}
\label{app:additional-experiments}

\subsection{Controlled Estimator Studies}
\label{appendix:controlled-experiments}

\paragraph{Controlled processes and exact references.}
We use a one-state Bernoulli-return process to fix the controlled-study
configuration and evaluate the frozen configuration on Solitaire Dice and a
discrete Monte Carlo chain. All three processes admit exact return laws, so
the finite-pool posterior target can be evaluated without a learned target
critic. Table~\ref{tab:controlled-setup} summarizes the process definitions.

\begin{table}[t]
    \centering
    \small
    \setlength{\tabcolsep}{3pt}
    \begin{tabular}{lp{0.62\columnwidth}}
        \toprule
        Process & Configuration \\
        \midrule
        Bernoulli return
        & One-state continuing MRP;
        \(R\sim\operatorname{Bernoulli}(0.5)\);
        \(\gamma=0.5\);
        exact return law \(Z\sim\operatorname{Unif}(0,2)\). \\

        Solitaire Dice
        & A fair six-sided die is repeatedly rolled.
        Rolls \(2\)--\(6\) produce reward \(1\) and continue;
        roll \(1\) produces reward \(0\) and terminates;
        \(\gamma=0.9\). \\

        Discrete-MC
        & Twenty-state nearest-neighbor chain with absorbing terminal states
        \(0\) and \(19\). States \(1\)--\(18\) use the fixed
        cosine--exponential nearest-neighbor transition weights of the
        controlled-process configuration. Transitions to nonterminal states
        produce reward \(1\), transitions into a terminal state produce
        reward \(0\), and \(\gamma=0.95\). \\
        \bottomrule
    \end{tabular}
    \caption{Controlled-process settings.
    The Bernoulli process is used to fix the controlled-study configuration;
    Solitaire Dice and Discrete-MC are used for the reported evaluation.}
    \label{tab:controlled-setup}
\end{table}

For a controlled conditional endpoint law \(\mu\), the exact posterior
Bellman velocity at bridge point \(x\) and flow time \(t\) is
\begin{equation}
    v^\star(x,t)
    =
    \frac{
        \mathbb{E}_{Y\sim\mu}
        \left[
            \ell_t(x\mid Y)u_t(x,Y)
        \right]
    }{
        \mathbb{E}_{Y\sim\mu}
        \left[\ell_t(x\mid Y)\right]
    },
    \label{eq:controlled-oracle}
\end{equation}
where \(\ell_t(x\mid Y)\) is the Gaussian bridge likelihood. The expectations
are evaluated from the exact endpoint laws. For each reported process, the
oracle study uses \(N_{\mathrm q}=4096\) bridge queries and \(N_{\mathrm{rep}}=512\) endpoint-pool
replicates per query. For query \(q\), we sample and fix
\((\tau_q,t_q,x_q)\) from the exact bridge marginal and compute the oracle
velocity \(v_q^\star\). In every replicate, the relabeled anchor is resampled
from the exact bridge posterior
\begin{equation}
    Y_{\mathrm{anc},q}^{(r)}
    \sim \nu_{\tau_q,x_q}(dy)
    \propto \ell_{t_q}(x_q\mid y)\,\mu_{\tau_q}(dy).
    \label{eq:controlled-posterior-anchor}
\end{equation}
\par\indent The remaining \(K-1\) endpoints are sampled independently from
\(\mu_{\tau_q}\). This is the conditional law of the uniform-anchor estimator
given the fixed bridge query, so the \(K=1\) target remains random and provides
a nonzero BCFM reference. Uniform aggregation and PA-BCFM use matched candidate
pools within each comparison. Process-level results give equal weight to
Solitaire Dice and Discrete-MC. Confidence intervals use a process-stratified
query-cluster bootstrap, retaining all replicates of each resampled query
together.

\paragraph{Fixed-query conditional target estimation.}
We evaluate the anchor-in-pool estimator with
\(K\in\{1,2,4,8,16\}\). For query \(q\) and endpoint-pool replicate \(r\), let
\(\widehat u_{K,q}^{(r)}\) denote the finite-pool target. We report the
conditional mean-squared error
\begin{equation}
    \operatorname{MSE}_K
    =
    \frac{1}{N_{\mathrm q}N_{\mathrm{rep}}}
    \sum_{q=1}^{N_{\mathrm q}}
    \sum_{r=1}^{N_{\mathrm{rep}}}
    \left(
        \widehat u_{K,q}^{(r)}
        -
        v^\star_q
    \right)^2.
    \label{eq:controlled-mse}
\end{equation}

\begin{table}[t]
    \centering
    \small
    \setlength{\tabcolsep}{5pt}
    \begin{tabular}{rccc}
        \toprule
        \(K\) & Signed bias [95\% CI] & Cond. variance & MSE [95\% CI] \\
        \midrule
        1  & \(0.004\,[-0.006,0.014]\)  & 5.0407 & \(5.0447\,[4.86,5.23]\) \\
        2  & \(-0.003\,[-0.011,0.005]\) & 3.0301 & \(3.0331\,[2.91,3.16]\) \\
        4  & \(0.002\,[-0.005,0.009]\)  & 1.7790 & \(1.7810\,[1.70,1.86]\) \\
        8  & \(-0.001\,[-0.006,0.004]\) & 1.0209 & \(1.0219\,[0.97,1.08]\) \\
        16 & \(0.001\,[-0.003,0.005]\)  & 0.5693 & \(0.5703\,[0.54,0.61]\) \\
        \bottomrule
    \end{tabular}
    \caption{Fixed-query conditional target estimation.
    Bias is the mean signed error; variance is the mean within-query conditional
    variance. Brackets give 95\% process-stratified query-cluster bootstrap
    confidence intervals. Lower MSE is better.}
    \label{tab:controlled-estimator-mse}

    \setlength{\tabcolsep}{4pt}
    \begin{tabular}{llccccc}
        \toprule
        Process & Estimator & \(K=1\) & \(K=2\) & \(K=4\) & \(K=8\) & \(K=16\) \\
        \midrule
        Solitaire
            & Uniform & \(\substack{1.000\\{[1.000,1.000]}}\)
            & \(\substack{.499\\{[.470,.529]}}\)
            & \(\substack{.227\\{[.208,.247]}}\)
            & \(\substack{.114\\{[.102,.126]}}\)
            & \(\substack{.058\\{[.051,.065]}}\) \\
        Solitaire
            & PA-BCFM & \(\substack{1.000\\{[1.000,1.000]}}\)
            & \(\substack{\mathbf{.438}\\{[.409,.468]}}\)
            & \(\substack{\mathbf{.179}\\{[.164,.195]}}\)
            & \(\substack{\mathbf{.081}\\{[.073,.090]}}\)
            & \(\substack{\mathbf{.038}\\{[.034,.043]}}\) \\
        Discrete-MC
            & Uniform & \(\substack{1.000\\{[1.000,1.000]}}\)
            & \(\substack{.460\\{[.432,.489]}}\)
            & \(\substack{.195\\{[.179,.212]}}\)
            & \(\substack{.098\\{[.089,.108]}}\)
            & \(\substack{.052\\{[.046,.058]}}\) \\
        Discrete-MC
            & PA-BCFM & \(\substack{1.000\\{[1.000,1.000]}}\)
            & \(\substack{\mathbf{.349}\\{[.322,.377]}}\)
            & \(\substack{\mathbf{.090}\\{[.080,.100]}}\)
            & \(\substack{\mathbf{.020}\\{[.017,.023]}}\)
            & \(\substack{\mathbf{.010}\\{[.0088,.0122]}}\) \\
        \bottomrule
    \end{tabular}
    \caption{Target-induced gradient trace-covariance ratio.
    Values are normalized by the corresponding random \(K=1\) BCFM reference;
    brackets give 95\% query-cluster bootstrap confidence intervals. Lower is
    better.}
    \label{tab:controlled-gradient-var}
\end{table}

\paragraph{Gradient variability and posterior concentration.}
For the gradient study, critic parameters and the bridge query
\((\tau,t,x_t)\) are fixed, while every replicate independently resamples the
posterior anchor and the remaining endpoint pool using the same conditional
protocol as above. Hence the \(K=1\) target and gradient remain random. For
gradient replicate \(g_r\), we compute the trace covariance
\begin{equation}
    \operatorname{GVar}
    =
    \operatorname{tr}
    \left(
        \widehat{\operatorname{Cov}}[g]
    \right)
    =
    \frac{1}{R_g}
    \sum_{r=1}^{R_g}
    \left\|
        g_r-\bar g
    \right\|_2^2.
    \label{eq:controlled-gradient-variance}
\end{equation}

For each query, we compute the trace covariance across \(R_g=256\) replicates.
The reported process-level ratio divides the aggregate query-level covariance
trace at pool size \(K\) by the corresponding \(K=1\) aggregate. These ratios
measure fixed-query, target-induced gradient variability rather than the
unconditional variance of full SGD updates. Confidence intervals use the same
query-cluster bootstrap.

For PA-BCFM, posterior concentration is summarized by the normalized
effective sample size
\begin{equation}
    \frac{\ESS(\mathbf w)}{K}
    =
    \frac{1}{
        K\sum_{j=1}^{K}w_j^2
    }.
    \label{eq:normalized-ess}
\end{equation}

ESS is used as a within-estimator measure of posterior concentration;
cross-estimator variance is assessed directly by
Table~\ref{tab:controlled-gradient-var}.

\paragraph{Learned return-law fidelity.}
Finally, we train BCFM, uniform endpoint aggregation, Value Flows, PCBF, and
PA-BCFM using the same critic architecture, optimizer, transition batches,
and \(50{,}000\)-update budget. Each configuration is trained with eight
seeds, and return-law fidelity is evaluated using \(50{,}000\) return samples
per evaluated state--action pair. Solitaire Dice is evaluated before the
first roll. For Discrete-MC, the process-level fidelity value aggregates the
configured statewise evaluation grid.

\begin{table}[ht]
    \centering
    \small
    \setlength{\tabcolsep}{2pt}
    \renewcommand{\arraystretch}{1.05}
    \begin{tabular}{lccc}
        \toprule
        Method & Solitaire & Discrete-MC & Avg. \\
        \midrule
        BCFM (\(K=1\))
            & \(\substack{6.644\\{[6.31,6.98]}}\)
            & \(\substack{1.454\\{[1.36,1.55]}}\)
            & \(\substack{4.049\\{[3.86,4.24]}}\) \\
        Uniform (\(K=8\))
            & \(\substack{6.252\\{[5.92,6.58]}}\)
            & \(\substack{3.630\\{[3.38,3.88]}}\)
            & \(\substack{4.941\\{[4.73,5.15]}}\) \\
        Value Flows
            & \(\substack{6.650\\{[6.25,7.05]}}\)
            & \(\substack{4.202\\{[3.91,4.49]}}\)
            & \(\substack{5.426\\{[5.18,5.67]}}\) \\
        PCBF (\(\lambda_{\mathrm{cv}}=0.3\))
            & \(\substack{6.020\\{[5.76,6.28]}}\)
            & \(\substack{1.227\\{[1.15,1.30]}}\)
            & \(\substack{3.623\\{[3.48,3.77]}}\) \\
        PA-BCFM (\(K=8\))
            & \(\substack{\mathbf{5.891}\\{[5.65,6.13]}}\)
            & \(\substack{\mathbf{1.185}\\{[1.12,1.25]}}\)
            & \(\substack{\mathbf{3.538}\\{[3.41,3.67]}}\) \\
        \bottomrule
    \end{tabular}
    \caption{Learned return-law fidelity.
    Entries are seed-level means with 95\% confidence intervals. Avg. is the
    equal-weight average over Solitaire Dice and Discrete-MC. Lower
    \(1\)-Wasserstein distance is better.}
    \label{tab:controlled-fidelity}
\end{table}

\FloatBarrier
\subsection{Full Offline OGBench Results}
\label{app:full-offline-results}

Tables~\ref{tab:offline-eval-state-full} and
\ref{tab:offline-eval-visual-full} report the complete task-level results
underlying Table~\ref{tab:offline-eval-agg}. Seed counts and the
domain-level tuning-task convention follow Appendix~\ref{app:evaluation-protocol}.

\begin{table}[p]
\begin{minipage}{\textwidth}
    \centering
    \small
    \setlength{\tabcolsep}{1pt}
    \begin{tabular}{lccccccccccc}
        \toprule
        & \multicolumn{3}{c}{Gaussian Policies}
        & \multicolumn{2}{c}{Flow Policies}
        & \multicolumn{6}{c}{Distributional Critics} \\
        \cmidrule(lr){2-4}
        \cmidrule(lr){5-6}
        \cmidrule(lr){7-12}
        Task
        & BC & IQL & ReBRAC
        & IFQL & FQL & C51 & IQN & CODAC & Value Flows & PCBF & \ours{} \\
        \midrule
        \multicolumn{12}{l}{\textit{\texttt{cube-double-play}}} \\
        \texttt{task1}
        & $8 \pm 3$
        & $27 \pm 5$
        & $45 \pm 6$
        & $35 \pm 9$
        & $61 \pm 9$
        & $9 \pm 0$
        & $70 \pm 14$
        & $80 \pm 11$
        & $97 \pm 1$
        & $92 \pm 3$
        & $\mathbf{99 \pm 1}$ \\

        \texttt{task2}\textsuperscript{*}
        & $0 \pm 0$
        & $1 \pm 1$
        & $7 \pm 3$
        & $9 \pm 5$
        & $36 \pm 6$
        & $0 \pm 0$
        & $24 \pm 9$
        & $63 \pm 4$
        & $76 \pm 7$
        & $74 \pm 7$
        & $\mathbf{89 \pm 2}$ \\

        \texttt{task3}
        & $0 \pm 0$
        & $0 \pm 0$
        & $4 \pm 1$
        & $8 \pm 5$
        & $22 \pm 5$
        & $0 \pm 0$
        & $25 \pm 6$
        & $66 \pm 9$
        & $73 \pm 4$
        & $81 \pm 8$
        & $\mathbf{91 \pm 2}$ \\

        \texttt{task4}
        & $0 \pm 0$
        & $0 \pm 0$
        & $1 \pm 1$
        & $1 \pm 1$
        & $5 \pm 2$
        & $0 \pm 0$
        & $10 \pm 1$
        & $13 \pm 2$
        & $30 \pm 5$
        & $22 \pm 5$
        & $\mathbf{44 \pm 7}$ \\

        \texttt{task5}
        & $0 \pm 0$
        & $4 \pm 3$
        & $4 \pm 2$
        & $17 \pm 6$
        & $19 \pm 10$
        & $0 \pm 0$
        & $81 \pm 8$
        & $82 \pm 4$
        & $69 \pm 5$
        & $\mathbf{84 \pm 3}$
        & $70 \pm 6$ \\

        \midrule
        \multicolumn{12}{l}{\textit{\texttt{cube-triple-play}}} \\
        \texttt{task1}\textsuperscript{*}
        & $1 \pm 1$
        & $4 \pm 4$
        & $1 \pm 2$
        & $2 \pm 2$
        & $20 \pm 6$
        & $1 \pm 0$
        & $29 \pm 2$
        & $9 \pm 5$
        & $59 \pm 12$
        & $18 \pm 4$
        & $\mathbf{60 \pm 8}$ \\

        \texttt{task2}
        & $0 \pm 0$
        & $0 \pm 0$
        & $0 \pm 0$
        & $0 \pm 0$
        & $\mathbf{1 \pm 2}$
        & $0 \pm 0$
        & $0 \pm 0$
        & $\mathbf{1 \pm 0}$
        & $0 \pm 0$
        & $0 \pm 1$
        & $0 \pm 0$ \\

        \texttt{task3}
        & $0 \pm 0$
        & $0 \pm 0$
        & $0 \pm 0$
        & $0 \pm 0$
        & $0 \pm 0$
        & $0 \pm 0$
        & $1 \pm 0$
        & $0 \pm 0$
        & $\mathbf{7 \pm 3}$
        & $1 \pm 1$
        & $\mathbf{7 \pm 2}$ \\

        \texttt{task4}
        & $0 \pm 0$
        & $0 \pm 0$
        & $0 \pm 0$
        & $0 \pm 0$
        & $0 \pm 0$
        & $0 \pm 0$
        & $0 \pm 0$
        & $0 \pm 0$
        & $0 \pm 0$
        & $0 \pm 0$
        & $0 \pm 0$ \\

        \texttt{task5}
        & $0 \pm 0$
        & $1 \pm 1$
        & $0 \pm 0$
        & $0 \pm 0$
        & $0 \pm 0$
        & $0 \pm 0$
        & $0 \pm 0$
        & $0 \pm 0$
        & $\mathbf{2 \pm 1}$
        & $1 \pm 1$
        & $1 \pm 1$ \\

        \midrule
        \multicolumn{12}{l}{\textit{\texttt{puzzle-3x3-play}}} \\
        \texttt{task1}
        & $5 \pm 2$
        & $33 \pm 6$
        & $97 \pm 4$
        & $94 \pm 3$
        & $90 \pm 4$
        & $5 \pm 0$
        & $71 \pm 3$
        & $78 \pm 8$
        & $99 \pm 0$
        & $73 \pm 2$
        & $\mathbf{100 \pm 0}$ \\

        \texttt{task2}
        & $1 \pm 1$
        & $4 \pm 3$
        & $1 \pm 1$
        & $1 \pm 2$
        & $16 \pm 5$
        & $0 \pm 0$
        & $2 \pm 2$
        & $5 \pm 2$
        & $98 \pm 2$
        & $23 \pm 5$
        & $\mathbf{100 \pm 0}$ \\

        \texttt{task3}
        & $1 \pm 1$
        & $3 \pm 2$
        & $3 \pm 1$
        & $0 \pm 0$
        & $10 \pm 3$
        & $0 \pm 0$
        & $0 \pm 0$
        & $4 \pm 3$
        & $\mathbf{97 \pm 1}$
        & $83 \pm 1$
        & $86 \pm 2$ \\

        \texttt{task4}\textsuperscript{*}
        & $1 \pm 1$
        & $2 \pm 1$
        & $2 \pm 1$
        & $0 \pm 0$
        & $16 \pm 5$
        & $0 \pm 0$
        & $0 \pm 0$
        & $5 \pm 5$
        & $84 \pm 24$
        & $42 \pm 5$
        & $\mathbf{94 \pm 4}$ \\

        \texttt{task5}
        & $1 \pm 0$
        & $3 \pm 2$
        & $5 \pm 3$
        & $0 \pm 0$
        & $16 \pm 3$
        & $0 \pm 0$
        & $0 \pm 0$
        & $6 \pm 5$
        & $58 \pm 39$
        & $43 \pm 4$
        & $\mathbf{87 \pm 8}$ \\

        \midrule
        \multicolumn{12}{l}{\textit{\texttt{puzzle-4x4-play}}} \\
        \texttt{task1}
        & $1 \pm 1$
        & $12 \pm 2$
        & $26 \pm 4$
        & $49 \pm 9$
        & $34 \pm 8$
        & $0 \pm 0$
        & $41 \pm 2$
        & $37 \pm 32$
        & $36 \pm 3$
        & $38 \pm 6$
        & $\mathbf{73 \pm 8}$ \\

        \texttt{task2}
        & $0 \pm 0$
        & $7 \pm 4$
        & $12 \pm 4$
        & $4 \pm 4$
        & $16 \pm 5$
        & $0 \pm 0$
        & $12 \pm 4$
        & $10 \pm 10$
        & $27 \pm 5$
        & $23 \pm 5$
        & $\mathbf{51 \pm 2}$ \\

        \texttt{task3}
        & $0 \pm 0$
        & $9 \pm 3$
        & $15 \pm 3$
        & $50 \pm 14$
        & $18 \pm 5$
        & $0 \pm 0$
        & $45 \pm 7$
        & $33 \pm 29$
        & $30 \pm 4$
        & $40 \pm 4$
        & $\mathbf{68 \pm 5}$ \\

        \texttt{task4}\textsuperscript{*}
        & $0 \pm 0$
        & $5 \pm 2$
        & $10 \pm 3$
        & $21 \pm 11$
        & $11 \pm 3$
        & $0 \pm 0$
        & $23 \pm 2$
        & $12 \pm 10$
        & $28 \pm 5$
        & $28 \pm 4$
        & $\mathbf{49 \pm 10}$ \\

        \texttt{task5}
        & $0 \pm 0$
        & $4 \pm 1$
        & $7 \pm 3$
        & $2 \pm 2$
        & $7 \pm 3$
        & $0 \pm 0$
        & $16 \pm 6$
        & $10 \pm 8$
        & $13 \pm 2$
        & $\mathbf{23 \pm 3}$
        & $20 \pm 6$ \\

        \midrule
        \multicolumn{12}{l}{\textit{\texttt{scene-play}}} \\
        \texttt{task1}
        & $19 \pm 6$
        & $94 \pm 3$
        & $95 \pm 2$
        & $98 \pm 3$
        & $\mathbf{100 \pm 0}$
        & $17 \pm 3$
        & $\mathbf{100 \pm 0}$
        & $99 \pm 0$
        & $99 \pm 0$
        & $\mathbf{100 \pm 0}$
        & $\mathbf{100 \pm 0}$ \\

        \texttt{task2}\textsuperscript{*}
        & $1 \pm 1$
        & $12 \pm 3$
        & $50 \pm 13$
        & $0 \pm 0$
        & $76 \pm 9$
        & $2 \pm 1$
        & $1 \pm 0$
        & $85 \pm 4$
        & $97 \pm 1$
        & $57 \pm 13$
        & $\mathbf{100 \pm 0}$ \\

        \texttt{task3}
        & $1 \pm 1$
        & $32 \pm 7$
        & $55 \pm 16$
        & $54 \pm 19$
        & $98 \pm 1$
        & $0 \pm 1$
        & $94 \pm 2$
        & $90 \pm 3$
        & $94 \pm 2$
        & $98 \pm 2$
        & $\mathbf{100 \pm 0}$ \\

        \texttt{task4}
        & $2 \pm 2$
        & $0 \pm 1$
        & $3 \pm 3$
        & $0 \pm 0$
        & $5 \pm 1$
        & $2 \pm 1$
        & $3 \pm 1$
        & $0 \pm 0$
        & $7 \pm 17$
        & $12 \pm 3$
        & $\mathbf{92 \pm 8}$ \\

        \texttt{task5}
        & $0 \pm 0$
        & $0 \pm 0$
        & $0 \pm 0$
        & $0 \pm 0$
        & $0 \pm 0$
        & $0 \pm 0$
        & $0 \pm 0$
        & $0 \pm 0$
        & $0 \pm 0$
        & $\mathbf{2 \pm 1}$
        & $0 \pm 0$ \\

        \bottomrule
    \end{tabular}
    \captionof{table}{Full offline evaluation on state-based OGBench tasks.
    Results cover 25 tasks and eight training seeds. An asterisk marks the
    domain-level tuning task. Bold denotes the highest displayed task mean;
    ties at the highest mean are all bold, while fully tied rows are left
    unbolded.}
    \label{tab:offline-eval-state-full}
\end{minipage}
\end{table}

\begin{table}[p]
\begin{minipage}{\textwidth}
    \centering
    \small
    \setlength{\tabcolsep}{2pt}
    \renewcommand{\arraystretch}{0.98}
    \begin{tabular}{lcccccccc}
        \toprule
        & \multicolumn{2}{c}{Gaussian Policies}
        & \multicolumn{2}{c}{Flow Policies}
        & \multicolumn{4}{c}{Distributional Critics} \\
        \cmidrule(lr){2-3}
        \cmidrule(lr){4-5}
        \cmidrule(lr){6-9}
        Task
        & IQL & ReBRAC
        & IFQL & FQL & IQN & Value Flows & PCBF & \ours{} \\
        \midrule
        \multicolumn{9}{l}{\textit{\texttt{visual-antmaze-medium-navigate}}} \\
        \texttt{task1}\textsuperscript{*}
        & $78 \pm 9$
        & $54 \pm 15$
        & $\mathbf{81 \pm 3}$
        & $32 \pm 3$
        & $62 \pm 7$
        & $77 \pm 4$
        & $65 \pm 3$
        & $78 \pm 7$ \\

        \texttt{task2}
        & $90 \pm 3$
        & $\mathbf{96 \pm 1}$
        & $87 \pm 1$
        & $60 \pm 2$
        & $88 \pm 2$
        & $75 \pm 5$
        & $73 \pm 1$
        & $92 \pm 2$ \\

        \texttt{task3}
        & $80 \pm 6$
        & $97 \pm 1$
        & $92 \pm 1$
        & $35 \pm 8$
        & $64 \pm 5$
        & $81 \pm 7$
        & $86 \pm 4$
        & $\mathbf{98 \pm 2}$ \\

        \texttt{task4}
        & $89 \pm 4$
        & $93 \pm 2$
        & $84 \pm 3$
        & $35 \pm 2$
        & $71 \pm 6$
        & $71 \pm 6$
        & $78 \pm 3$
        & $\mathbf{96 \pm 3}$ \\

        \texttt{task5}
        & $84 \pm 2$
        & $\mathbf{97 \pm 1}$
        & $89 \pm 2$
        & $29 \pm 7$
        & $84 \pm 1$
        & $70 \pm 26$
        & $71 \pm 1$
        & $93 \pm 1$ \\

        \midrule
        \multicolumn{9}{l}{\textit{\texttt{visual-antmaze-teleport-navigate}}} \\
        \texttt{task1}\textsuperscript{*}
        & $5 \pm 2$
        & $2 \pm 0$
        & $7 \pm 4$
        & $2 \pm 1$
        & $2 \pm 1$
        & $10 \pm 4$
        & $8 \pm 2$
        & $\mathbf{21 \pm 9}$ \\

        \texttt{task2}
        & $10 \pm 2$
        & $10 \pm 3$
        & $13 \pm 3$
        & $6 \pm 1$
        & $7 \pm 3$
        & $17 \pm 5$
        & $19 \pm 4$
        & $\mathbf{33 \pm 8}$ \\

        \texttt{task3}
        & $7 \pm 7$
        & $4 \pm 1$
        & $8 \pm 9$
        & $9 \pm 4$
        & $6 \pm 4$
        & $16 \pm 3$
        & $\mathbf{18 \pm 4}$
        & $16 \pm 2$ \\

        \texttt{task4}
        & $4 \pm 6$
        & $4 \pm 0$
        & $18 \pm 2$
        & $9 \pm 1$
        & $4 \pm 2$
        & $16 \pm 5$
        & $18 \pm 5$
        & $\mathbf{32 \pm 3}$ \\

        \texttt{task5}
        & $2 \pm 1$
        & $2 \pm 1$
        & $4 \pm 2$
        & $1 \pm 1$
        & $2 \pm 1$
        & $8 \pm 2$
        & $8 \pm 3$
        & $\mathbf{10 \pm 2}$ \\

        \midrule
        \multicolumn{9}{l}{\textit{\texttt{visual-cube-double-play}}} \\
        \texttt{task1}\textsuperscript{*}
        & $34 \pm 23$
        & $4 \pm 4$
        & $8 \pm 6$
        & $23 \pm 4$
        & $4 \pm 1$
        & $\mathbf{35 \pm 2}$
        & $10 \pm 3$
        & $30 \pm 10$ \\

        \texttt{task2}
        & $3 \pm 1$
        & $0 \pm 0$
        & $0 \pm 0$
        & $0 \pm 0$
        & $0 \pm 0$
        & $\mathbf{4 \pm 2}$
        & $0 \pm 0$
        & $0 \pm 0$ \\

        \texttt{task3}
        & $7 \pm 4$
        & $2 \pm 2$
        & $1 \pm 1$
        & $4 \pm 2$
        & $0 \pm 0$
        & $\mathbf{11 \pm 2}$
        & $0 \pm 0$
        & $10 \pm 4$ \\

        \texttt{task4}
        & $2 \pm 1$
        & $0 \pm 0$
        & $0 \pm 0$
        & $0 \pm 0$
        & $0 \pm 0$
        & $2 \pm 1$
        & $0 \pm 0$
        & $\mathbf{10 \pm 4}$ \\

        \texttt{task5}
        & $11 \pm 2$
        & $0 \pm 0$
        & $2 \pm 1$
        & $4 \pm 1$
        & $1 \pm 1$
        & $13 \pm 3$
        & $3 \pm 1$
        & $\mathbf{33 \pm 9}$ \\

        \midrule
        \multicolumn{9}{l}{\textit{\texttt{visual-puzzle-3x3-play}}} \\
        \texttt{task1}\textsuperscript{*}
        & $7 \pm 15$
        & $88 \pm 4$
        & $\mathbf{100 \pm 0}$
        & $94 \pm 1$
        & $84 \pm 1$
        & $93 \pm 5$
        & $81 \pm 1$
        & $90 \pm 2$ \\

        \texttt{task2}
        & $0 \pm 0$
        & $\mathbf{12 \pm 1}$
        & $0 \pm 0$
        & $0 \pm 0$
        & $6 \pm 2$
        & $\mathbf{12 \pm 1}$
        & $0 \pm 0$
        & $0 \pm 0$ \\

        \texttt{task3}
        & $0 \pm 0$
        & $1 \pm 1$
        & $2 \pm 1$
        & $0 \pm 0$
        & $1 \pm 0$
        & $3 \pm 1$
        & $2 \pm 1$
        & $\mathbf{5 \pm 1}$ \\

        \texttt{task4}
        & $1 \pm 1$
        & $0 \pm 1$
        & $1 \pm 0$
        & $5 \pm 4$
        & $3 \pm 0$
        & $\mathbf{6 \pm 2}$
        & $5 \pm 3$
        & $4 \pm 1$ \\

        \texttt{task5}
        & $0 \pm 0$
        & $0 \pm 0$
        & $0 \pm 0$
        & $1 \pm 0$
        & $1 \pm 0$
        & $\mathbf{2 \pm 0}$
        & $1 \pm 0$
        & $\mathbf{2 \pm 1}$ \\

        \midrule
        \multicolumn{9}{l}{\textit{\texttt{visual-scene-play}}} \\
        \texttt{task1}\textsuperscript{*}
        & $97 \pm 2$
        & $98 \pm 4$
        & $86 \pm 10$
        & $98 \pm 3$
        & $95 \pm 2$
        & $99 \pm 0$
        & $40 \pm 4$
        & $\mathbf{100 \pm 1}$ \\

        \texttt{task2}
        & $21 \pm 16$
        & $30 \pm 15$
        & $0 \pm 0$
        & $\mathbf{86 \pm 8}$
        & $79 \pm 15$
        & $40 \pm 27$
        & $62 \pm 8$
        & $42 \pm 14$ \\

        \texttt{task3}
        & $12 \pm 9$
        & $10 \pm 7$
        & $19 \pm 2$
        & $22 \pm 6$
        & $31 \pm 14$
        & $66 \pm 1$
        & $35 \pm 1$
        & $\mathbf{68 \pm 12}$ \\

        \texttt{task4}
        & $1 \pm 0$
        & $0 \pm 0$
        & $0 \pm 0$
        & $1 \pm 1$
        & $0 \pm 0$
        & $10 \pm 6$
        & $12 \pm 4$
        & $\mathbf{38 \pm 11}$ \\

        \texttt{task5}
        & $0 \pm 0$
        & $0 \pm 0$
        & $0 \pm 0$
        & $0 \pm 0$
        & $0 \pm 0$
        & $0 \pm 0$
        & $0 \pm 0$
        & $0 \pm 0$ \\

        \bottomrule
    \end{tabular}
    \captionof{table}{Full offline evaluation on 25 image-based OGBench tasks
    over four seeds. An asterisk marks the tuning task; bold denotes the
    highest displayed task mean, with fully tied rows left unbolded.}
    \label{tab:offline-eval-visual-full}
\end{minipage}
\end{table}

\begingroup
\raggedbottom
\subsection{Offline-to-Online Results}
\label{app:o2o-results}

Table~\ref{tab:o2o-full} reports scenario-level performance at online entry
and at the final online checkpoint.
\begin{table}[!t]
\centering
\small
\setlength{\tabcolsep}{1.5pt}
\begin{tabular}{lcccc}
\toprule
Method & Antmaze & Cube-D & Cube-T & Scene \\
\midrule
\multicolumn{5}{c}{\textbf{Offline checkpoint at \(1000\text{k}\) steps}} \\
\midrule
IQN         & \(55{\pm}5\) & \(29{\pm}4\) & \(25{\pm}5\) & \(0{\pm}0\) \\
IFQL        & \(11{\pm}3\) & \(12{\pm}3\) & \(2{\pm}1\) & \(0{\pm}0\) \\
FQL         & \(\mathbf{87{\pm}7}\) & \(40{\pm}11\) & \(4{\pm}1\) & \(\mathbf{82{\pm}11}\) \\
RLPD        & \(0{\pm}0\) & \(0{\pm}0\) & \(0{\pm}0\) & \(0{\pm}0\) \\
IQL         & \(41{\pm}5\) & \(0{\pm}0\) & \(3{\pm}2\) & \(15{\pm}4\) \\
Value Flows & \(72{\pm}4\) & \(65{\pm}7\) & \(\mathbf{29{\pm}8}\) & \(\mathbf{83{\pm}4}\) \\
\ours{}     & \(\mathbf{89{\pm}4}\) & \(\mathbf{78{\pm}4}\) & \(24{\pm}2\) & \(\mathbf{79{\pm}3}\) \\
\midrule
\multicolumn{5}{c}{\textbf{Online final at \(2000\text{k}\) steps}} \\
\midrule
IQN         & \(91{\pm}1\) & \(42{\pm}7\) & \(14{\pm}6\) & \(0{\pm}0\) \\
IFQL        & \(3{\pm}2\) & \(41{\pm}2\) & \(7{\pm}5\) & \(60{\pm}14\) \\
FQL         & \(\mathbf{99{\pm}1}\) & \(\mathbf{92{\pm}3}\) & \(\mathbf{83{\pm}12}\) & \(\mathbf{100{\pm}1}\) \\
RLPD        & \(3{\pm}2\) & \(0{\pm}0\) & \(0{\pm}0\) & \(\mathbf{100{\pm}0}\) \\
IQL         & \(45{\pm}5\) & \(0{\pm}0\) & \(0{\pm}0\) & \(10{\pm}3\) \\
Value Flows & \(\mathbf{96{\pm}1}\) & \(79{\pm}6\) & \(70{\pm}7\) & \(\mathbf{100{\pm}0}\) \\
\ours{}     & \(\mathbf{99{\pm}1}\) & \(\mathbf{95{\pm}2}\) & \(\mathbf{85{\pm}4}\) & \(\mathbf{100{\pm}0}\) \\
\bottomrule
\end{tabular}
\caption{Offline-to-online evaluation on four OGBench scenarios.
The two panels report success at online entry and after one million online
environment steps. Entries are means \(\pm\) standard deviations over eight
seeds; bold
denotes means within \(95\%\) of the highest scenario mean.}
\label{tab:o2o-full}
\end{table}

\subsection{Critic Diagnostics and Adaptive Pessimism}
\label{app:pool-diagnostics}

We evaluate critic reliability at the final \(10^6\)-update checkpoint on
\texttt{puzzle-4x4-play-task4} using eight seeds. This prespecified
domain-level tuning task is not selected using the diagnostic results. For
each of \(N_q=128\) held-out state--action pairs, we construct
\begin{equation}
    \begin{aligned}
        a_\delta
        &=
        \operatorname{clip}\!\left(
            (1-\delta)a_{\mathcal D}
            +\delta a_{\mathrm{rand}},-1,1
        \right),\\[-2pt]
        \delta
        &\in\{0,.2,.4,.6,.8,1\}.
    \end{aligned}
    \label{eq:appendix-controlled-action-shift}
\end{equation}
where \(a_{\mathrm{rand}}\) is sampled uniformly from the normalized action
space and shared across methods. Absolute prediction error compares the
twin-critic mean with a Monte Carlo return averaged over
\(N_{\mathrm{ref}}=64\) rollouts after forcing \(a_\delta\) as the first
action and then following the checkpoint's evaluation policy. The assessed
critic is not used in this reference. We compute \(D_{\mathrm{twin}}\) from
\(M_{\mathrm{diag}}=64\) shared return-noise samples and define
\(P_{\mathrm{out}}\) as predicted mass above the 95th percentile of
    offline-dataset return-to-go. Prediction error,
    \(\sqrt{D_{\mathrm{twin}}}\), and positive overestimation are normalized by
    the difference between the 95th and 5th return-to-go percentiles;
    unsupported mass is already dimensionless.
Metrics are averaged over queries and seeds.

Across the six action-shift levels, both diagnostics track absolute
return-prediction error. The descriptive Spearman correlations are
\(\rho_{\mathrm{twin}}=0.943\) for twin-critic disagreement and
\(\rho_{\mathrm{out}}=0.886\) for unsupported mass.

For conservative training, we compare the undistorted target, fixed soft
tilts \(\eta\in\{0.1,0.3,0.6\}\), the adaptive rule in
Eq.~\eqref{eq:adaptive-distortion}, and Value Flows. The adaptive rule uses
\(D_{\mathrm{pool}}\) computed from undistorted posterior weights, clips
normalized disagreement to \([0,1]\), and caps the tilt at \(\eta_0=0.6\).
It is a stop-gradient practical heuristic and is not directly covered by the
fixed-distortion contraction analysis in Appendix~\ref{app:distorted-laws}.
Table~\ref{tab:diagnostic-distortion} reports the complete fixed-strength
sweep. Stronger fixed tilting reduces positive overestimation but eventually
degrades policy performance. The adaptive variant reduces normalized positive
overestimation from \(0.380\) to \(0.165\) while increasing final success from
\(0.490\) to \(0.515\).
\begin{table}[H]
\centering
\small
\setlength{\tabcolsep}{4pt}
\renewcommand{\arraystretch}{1.05}
\begin{tabular}{lcc}
\toprule
Variant
& \shortstack{Positive\\overestimation \(\downarrow\)}
& \shortstack{Final success\\\(\uparrow\)} \\
\midrule
No distortion & \(0.380 \pm 0.033\) & \(0.490 \pm 0.100\) \\
Fixed \(\eta=0.1\) & \(0.300 \pm 0.027\) & \(0.505 \pm 0.090\) \\
Fixed \(\eta=0.3\) & \(0.220 \pm 0.024\) & \(0.455 \pm 0.110\) \\
Fixed \(\eta=0.6\) & \(\mathbf{0.105 \pm 0.019}\) & \(0.365 \pm 0.130\) \\
Adaptive & \(0.165 \pm 0.023\) & \(\mathbf{0.515 \pm 0.090}\) \\
Value Flows & \(0.335 \pm 0.041\) & \(0.280 \pm 0.050\) \\
\bottomrule
\end{tabular}
\caption{Conservatism--performance trade-off on
\texttt{puzzle-4x4-play-task4}.
The undistorted PA-BFM and Value Flows scores correspond to the task-level
results in the main benchmark table. Lower positive overestimation and
higher final success are better.}
\label{tab:candidate-pool-summary}
\label{tab:diagnostic-distortion}
\end{table}
\endgroup

\subsection{Runtime Analysis}
\label{app:estimator-ablation}

Component and pool-size ablations are shown in
Figure~\ref{fig:estimator-ablation} of the main paper. The corresponding
final-score differences are reported in the main text; here we focus on
computational cost.

\paragraph{Computational cost.}
Table~\ref{tab:runtime-comparison} reports wall-clock hours for
\(10^6\)-update budgets on one representative state-based configuration and
one representative visual configuration. These are configuration-level
values rather than averages over all benchmark domains. All methods use one
NVIDIA GeForce RTX 4090 with 24\,GB of memory, batch size
\(256\), and the same implementation, compilation, and data-loading pipeline.
The same end-to-end timing boundary is applied to every
method within a column; endpoint candidates within an update are evaluated in
parallel and reuse the successor-state representation.

\begin{table}[H]
    \centering
    \small
    \setlength{\tabcolsep}{2pt}
    \begin{tabular}{lcrr}
        \toprule
        Method & Target work & State & Visual \\
        \midrule
        BCFM (\(K=1\))
            & 2 target solves
            & 3.6 & 10.5 \\
        Uniform (\(K=8\))
            & 16 target solves
            & 3.8 & 11.0 \\
        PA-BCFM (\(K=4\))
            & 8 solves + wt.
            & 3.8 & 11.0 \\
        PA-BCFM (\(K=8\))
            & 16 solves + wt.
            & 4.0 & 11.8 \\
        PA-BCFM (\(K=16\))
            & 32 solves + wt.
            & 4.5 & 13.2 \\
        Value Flows
            & path + confidence
            & 4.4 & 12.5 \\
        PCBF (\(\lambda_{\mathrm{cv}}=0.3\))
            & endpoint + CV
            & 4.0 & 11.5 \\
        \bottomrule
    \end{tabular}
    \caption{Wall-clock training cost.
    State and visual columns report configuration-level hours for
    \(10^6\)-update budgets on one NVIDIA GeForce RTX 4090 (24\,GB) under the
    two representative timing settings. Target work counts both target
    critics, so \(K\) ensemble-mean endpoints require \(2K\) target-flow
    solves.}
    \label{tab:runtime-comparison}
\end{table}

The default \(K=8\) configuration remains below the measured Value Flows
runtime in both settings; \(K=16\) adds cost with little gain in
Figure~\ref{fig:estimator-ablation}, supporting \(K=8\) as the default.

\FloatBarrier
\section{Experimental Details}
\label{app:experiment-details}

All methods are implemented and evaluated in a common JAX codebase adapted
from FQL and Value Flows~\citep{park2025flow,dong2026valueflows}. Methods
share OGBench wrappers, dataset revisions, preprocessing, budgets, and
evaluation code.
BCFM, Value Flows, PCBF, and PA-BFM additionally share the actor, behavior
proposal, critic backbone, target update, ODE solver, schedule, and mean-based
action extraction. They differ in critic supervision: endpoint matching,
weighted consistency, a path-coupled control variate, and endpoint-pool
posterior averaging, respectively. Isolation claims apply only to this group.

\subsection{Environments}
\label{app:environments}

\paragraph{Controlled environments.}
The one-state Bernoulli-return process is a development process used for
configuration selection. Solitaire Dice and the discrete Monte Carlo chain
are the controlled evaluation processes.
The complete process definitions are given in
Table~\ref{tab:controlled-setup}. All three processes admit exact return laws,
so the reference return distributions and posterior Bellman velocities are
computed from the exact endpoint laws; no Monte Carlo reference pool is used
for the controlled results.

For Discrete-MC, let \(n=20\), \(\mathcal X=\{0,\ldots,n-1\}\), and define
the unnormalized state weight
\begin{equation}
    \omega_i
    =\exp\!\left[
        \frac{n-1}{4\pi}
        \cos\!\left(\frac{4\pi(i-1)}{n-1}\right)
    \right].
    \label{eq:discrete-mc-state-weight}
\end{equation}
\par\indent For each nonterminal state \(i\in\{1,\ldots,n-2\}\), the transition kernel is
\begin{equation}
    \begin{aligned}
    P(i,i\mathbin{\pm}1)
    &=\frac12\frac{\omega_{i\pm1}}{\omega_i+\omega_{i\pm1}},\\
    P(i,i)&=1-P(i,i-1)-P(i,i+1),
    \end{aligned}
    \label{eq:discrete-mc-kernel}
\end{equation}
with \(P(0,0)=P(n-1,n-1)=1\). Episodes initialize from
\(X_0\sim\operatorname{Unif}\{1,\ldots,n-2\}\). A transition receives reward
\(R_h=\mathbf 1\{X_{h+1}\notin\{0,n-1\}\}\), and \(\gamma=0.95\).

\paragraph{Offline OGBench.}
We evaluate the 50 single-task state and image environments listed in
Tables~\ref{tab:offline-eval-state-full}--\ref{tab:offline-eval-visual-full}
using the standard OGBench protocol
~\citep{dong2026valueflows,park2025ogbenchbenchmark}. State-based tasks use
native observations; image-based tasks use three stacked RGB frames, a small
IMPALA encoder, and random-crop augmentation.

\paragraph{Offline-to-online OGBench.}
We additionally evaluate offline-to-online fine-tuning on
\texttt{antmaze-large}, \texttt{cube-double}, \texttt{cube-triple}, and
\texttt{scene}. Each method is first trained for one million offline optimizer
updates and is then fine-tuned for one million online environment steps while
newly collected transitions are appended to the replay buffer. Each method
retains its own critic objective throughout both phases.

\subsection{Methods for Comparison}
\label{app:baselines}

The offline tables compare Gaussian-policy scalar critics (BC, IQL, and
ReBRAC), flow-policy scalar critics (IFQL and FQL), distributional critics
(C51, IQN, and CODAC), and flow-based distributional critics (Value Flows,
PCBF, and PA-BFM)~\citep{kostrikov2021,tarasov2023revisitingminimalist,
park2025flow,bellemare2017distributional,dabney2018iqn,
ma2021conservative,dong2026valueflows,xu2026pcbf}. PCBF also appears in the
controlled and runtime comparisons but not in offline-to-online evaluation.
The latter additionally includes RLPD under the common interaction budget
and checkpoint schedule~\citep{ball2023rlpd}. PA-BFM, Value Flows, and PCBF
share the policy-side implementation in the matched comparison and differ in
critic supervision.

\subsection{Evaluation Protocol}
\label{app:evaluation-protocol}

\paragraph{Offline evaluation.}
State-based and image-based tasks use eight and four training seeds,
respectively. Each evaluation checkpoint contains 50 episodes, and the main
tables report the final checkpoint. For each domain, the reported mean is the
arithmetic mean of the five task means, and the displayed uncertainty is the
arithmetic mean of the corresponding five task-level standard deviations.
The task marked with an asterisk is used for domain-level hyperparameter
selection under a common search budget for the matched methods. The selected
configuration is fixed across the remaining tasks in that domain, while the
tuning task remains in the five-task aggregate for comparability with prior
work. Task-level win/tie/loss counts compare final means and are descriptive
rather than significance tests.

\paragraph{Offline-to-online evaluation.}
Online interaction begins from the one-million-optimizer-update offline
checkpoint.
Methods are evaluated every 100k environment steps through 2000k total steps.
Each method uses eight independent training seeds. At each checkpoint, the
learning curves show the mean success rate over training seeds, and shaded
bands show one standard deviation across training seeds. Training seeds and
evaluation environment seeds are disjoint, and evaluation seeds are shared
across methods within each scenario.

\paragraph{Controlled evaluation.}
The controlled posterior-target and gradient-covariance protocols are
specified with their corresponding metrics in
Appendix~\ref{appendix:controlled-experiments}. Paired estimator comparisons
reuse the same candidate endpoint pools whenever the estimators share an
endpoint budget.
Controlled target-MSE and gradient-covariance intervals use a
process-stratified query-cluster bootstrap, with all replicates of a selected
query retained together. Learned-fidelity intervals use training seeds as the
resampling unit, and paired method differences use matched seed indices. All
reported intervals are 95\% confidence intervals.

\subsection{Architectures and Hyperparameters}
\label{appendix-hyperparameters}

The network architectures and optimization settings of the matched
flow-critic methods follow Value Flows~\citep{dong2026valueflows}.
Table~\ref{tab:common-hyperparameters} summarizes the settings shared by
PA-BFM and the corresponding flow-critic baselines. Entries involving flow
integration or candidate-action selection apply only to methods containing
those components.

\begin{table}[t]
    \centering
    \small
    \setlength{\tabcolsep}{3pt}
    \begin{tabular}{p{0.56\columnwidth}p{0.34\columnwidth}}
        \toprule
        \textbf{Hyperparameter} & \textbf{Value} \\
        \midrule
        Optimizer
            & Adam~\citep{kingma2014adam} \\
        Batch size
            & \(256\) \\
        Critic learning rate
            & \(3\times10^{-4}\) \\
        Critic hidden-layer sizes
            & \((512,512,512,512)\) \\
        Activation function
            & GELU \\
        Actor layer normalization
            & Yes \\
        Critic/value layer normalization
            & Yes \\
        Return-flow Euler steps
            & \(10\) \\
        Return samples per candidate action \(M_{\mathrm{act}}\)
            & \(1\) \\
        Candidate actions \(N\)
            & \(16\) \\
        Target-network update coefficient \(\tau_{\mathrm{tgt}}\)
            & \(5\times10^{-3}\) \\
        Number of critics
            & \(2\) \\
        Image encoder
            & Small IMPALA~\citep{park2025flow} \\
        Image augmentation
            & Random cropping \\
        Image frame stack
            & \(3\) \\
        Numeric precision
            & \texttt{float32} \\
        \bottomrule
    \end{tabular}
    \caption{Common hyperparameters for PA-BFM and matched flow-critic
    baselines. Entries apply where the corresponding component is present.}
    \label{tab:common-hyperparameters}
\end{table}

\paragraph{PA-BFM-specific hyperparameters.}
The default PA-BCFM critic uses \(K=8\) Bellman endpoints and no pessimistic
distortion, i.e., \(g\equiv1\). Bridge time is sampled as
\(t\sim\operatorname{Unif}(0,1-10^{-5})\), and posterior logits are
normalized by log-sum-exp. All downstream experiments use this setting;
\(K\in\{1,2,4,8,16\}\) is varied only in the controlled estimator and
pool-sensitivity studies.

\paragraph{Baseline hyperparameter selection.}
IQL selects its advantage-weighting inverse temperature from
\(\{0.3,1,3,10\}\). ReBRAC selects its actor and critic regularization
coefficients from \(\{0.003,0.01,0.03,0.1,0.3,1\}\) and
\(\{0,0.001,0.01,0.1\}\), respectively. IFQL selects the number of proposed
actions from \(\{32,64,128\}\), and FQL uses its original regularization grid.
C51 uses 51 or 101 atoms. IQN and CODAC select the quantile-loss parameter
from \(\{0.7,0.8,0.9,0.95\}\), with CODAC penalty \(0.1\). PCBF uses
\(\lambda_{\mathrm{cv}}=0.3\), selected on the Bernoulli development process
and frozen for Solitaire Dice and Discrete-MC. Each selected setting is fixed
for the remaining tasks in its group.

\subsection{Runtime and Reproducibility}
\label{app:reproducibility}

All experiments use JAX on one NVIDIA GeForce RTX 4090 (24\,GB), with batch
size 256 and common implementation, compilation, data-loading, and evaluation
boundaries. Table~\ref{tab:runtime-comparison} reports end-to-end wall-clock
hours for one million updates on representative state and visual
configurations; these are configuration-level measurements, not domain
averages. Performance comparisons use matched optimizer-update counts.

We record the software and hardware versions, OGBench and dataset revisions,
random seeds, training commands, configuration overrides,
checkpoint-selection rules, and evaluation commands. Tables and figures use
the same plotting and aggregation pipeline across methods.

\begin{table}[H]
    \centering
    \small
    \setlength{\tabcolsep}{4pt}
    \begin{tabular}{@{}p{0.56\columnwidth}lc@{}}
        \toprule
        Domain/scenario & Evaluation & \(\gamma\) \\
        \midrule
        \texttt{cube-double} & Off.+O2O & .995 \\
        \texttt{cube-triple} & Off.+O2O & .995 \\
        \texttt{puzzle-3x3} & Offline & .99 \\
        \texttt{puzzle-4x4} & Offline & .99 \\
        \texttt{scene} & Off.+O2O & .99 \\
        \texttt{visual-antmaze-}\newline\texttt{medium}
            & Offline & .99 \\
        \texttt{visual-antmaze-}\newline\texttt{teleport}
            & Offline & .99 \\
        \texttt{visual-cube-double} & Offline & .995 \\
        \texttt{visual-puzzle-3x3} & Offline & .99 \\
        \texttt{visual-scene} & Offline & .99 \\
        \texttt{antmaze-large} & O2O only & .99 \\
        \bottomrule
    \end{tabular}
    \caption{Domain/scenario settings.
    PA-BFM uses an undistorted \(K=8\) critic, and all main experiments use
    mean-based candidate scoring.}
    \label{tab:domain-hyperparameters}
\end{table}

\end{document}